\newif\ifSHORTVERSION
\SHORTVERSIONfalse

\ifSHORTVERSION
  \documentclass[letterpaper,10pt,conference]{ieeeconf}
  \IEEEoverridecommandlockouts
\else
  \documentclass[journal,10pt,twoside]{IEEEtran}
\fi

\usepackage[T1]{fontenc}
\usepackage[utf8]{inputenc}

\ifSHORTVERSION
  \usepackage{amsmath,amssymb,amsfonts,mathtools}

  \usepackage{amsthm}
\else
  \usepackage{amsmath,amssymb,amsfonts,amsthm,mathtools}
\fi

\usepackage{newtxtext,newtxmath}
\usepackage{microtype}
\usepackage{bm}

\usepackage{xcolor}
\usepackage{tikz}

\usetikzlibrary{
  arrows.meta,
  positioning,
  calc,
  fit,
  backgrounds
}

\let\labelindent\relax
\usepackage{enumitem}

\usepackage{aliascnt}

\ifSHORTVERSION
\else
  \usepackage[
    a4paper,
    left=18mm,
    right=18mm,
    top=24mm,
    bottom=22mm,
    columnsep=6mm
  ]{geometry}
\fi

\newtheorem{definition}{Definition}

\newaliascnt{proposition}{definition}
\newtheorem{proposition}[proposition]{Proposition}
\aliascntresetthe{proposition}

\newaliascnt{theorem}{definition}
\newtheorem{theorem}[theorem]{Theorem}
\aliascntresetthe{theorem}

\newaliascnt{corollary}{definition}
\newtheorem{corollary}[corollary]{Corollary}
\aliascntresetthe{corollary}

\newaliascnt{assumption}{definition}
\newtheorem{assumption}[assumption]{Assumption}
\aliascntresetthe{assumption}

\theoremstyle{remark}
\newaliascnt{remark}{definition}
\newtheorem{remark}[remark]{Remark}
\aliascntresetthe{remark}

\theoremstyle{remark}
\newaliascnt{example}{definition}
\newtheorem{example}[example]{Example}
\aliascntresetthe{example}

\usepackage{graphicx}
\graphicspath{{figures/}{figs/}}

\usepackage{booktabs}
\usepackage{array}
\usepackage{pgfplots}

\pgfplotsset{compat=1.18}

\usetikzlibrary{
  arrows.meta,
  calc,
  positioning,
  decorations.pathreplacing
}

\ifSHORTVERSION
\else
  \usepackage[font=footnotesize]{caption}
  \usepackage[font=footnotesize]{subcaption}
\fi

\usepackage{cite}
\usepackage[hidelinks]{hyperref}
\usepackage[nameinlink,noabbrev]{cleveref}

\crefformat{equation}{#2(#1)#3}
\Crefformat{equation}{Equation~#2(#1)#3}

\crefname{figure}{Fig.}{Figs.}
\Crefname{figure}{Figure}{Figures}

\crefname{table}{Table}{Tables}
\Crefname{table}{Table}{Tables}

\crefname{section}{Sec.}{Secs.}
\Crefname{section}{Section}{Sections}

\crefname{subsection}{Sec.}{Secs.}
\Crefname{subsection}{Section}{Sections}

\crefname{definition}{Definition}{Definitions}
\Crefname{definition}{Definition}{Definitions}

\crefname{proposition}{Proposition}{Propositions}
\Crefname{proposition}{Proposition}{Propositions}

\crefname{theorem}{Theorem}{Theorems}
\Crefname{theorem}{Theorem}{Theorems}

\crefname{corollary}{Corollary}{Corollaries}
\Crefname{corollary}{Corollary}{Corollaries}

\crefname{assumption}{Assumption}{Assumptions}
\Crefname{assumption}{Assumption}{Assumptions}

\crefname{remark}{Remark}{Remarks}
\Crefname{remark}{Remark}{Remarks}
\crefname{example}{Example}{Examples}
\Crefname{example}{Example}{Examples}

\ifSHORTVERSION
\else
  \setlist[enumerate]{
    leftmargin=*,
    itemsep=0.3ex,
    topsep=0.5ex,
    parsep=0pt
  }

  \usepackage{titlesec}

  \renewcommand{\thesection}{\arabic{section}}
  \renewcommand{\thesubsection}
    {\thesection.\arabic{subsection}}
  \renewcommand{\thesubsubsection}
    {\thesubsection.\arabic{subsubsection}}

  \titleformat{\section}
    {\normalfont\large\bfseries}
    {\thesection}{0.65em}{}

  \titleformat{\subsection}
    {\normalfont\normalsize\bfseries}
    {\thesubsection}{0.65em}{}

  \titleformat{\subsubsection}
    {\normalfont\normalsize\itshape}
    {\thesubsubsection}{0.65em}{}

  \titlespacing*{\section}
    {0pt}{2.0ex plus 0.5ex minus 0.2ex}{0.8ex}

  \titlespacing*{\subsection}
    {0pt}{1.5ex plus 0.4ex minus 0.2ex}{0.5ex}

  \titlespacing*{\subsubsection}
    {0pt}{1.2ex plus 0.3ex minus 0.2ex}{0.4ex}

\fi

\newif\ifanonymous
\anonymousfalse

\newcommand{\papertitle}{%
\huge Emergent Underactuation in Fully Actuated Multirotors by Minimizing a Homogeneous Allocation Cost}

\newcommand{\PaperAuthorsAffiliated}{%
Antonio Franchi\textsuperscript{1,2}
 and Barbara Bazzana\textsuperscript{1}%
}

\newcommand{\AffiliationOne}{%
\textsuperscript{1}Robotics and Mechatronics,
Faculty of Electrical Engineering, Mathematics and Computer Science,
University of Twente, Enschede, The Netherlands.%
}

\newcommand{\AffiliationTwo}{%
\textsuperscript{2}Department of Computer, Control and Management
Engineering, Sapienza University of Rome, Rome, Italy.%
}

\newcommand{\CorrespondingAuthor}{%
\textit{Corresponding author:} Antonio Franchi,
\href{mailto:schol@r-franchi.eu}{schol@r-franchi.eu}.%
}

\newcommand{\FundingStatement}{%
This work was partially funded by the Horizon Europe
research project AUTOASSESS under grant agreement
No.~101120732.%
}

\newcommand{\PaperAbstract}{%
Full actuation is commonly valued for independent position and attitude
tracking. This work establishes a complementary principle: when only a
position trajectory is prescribed, minimizing a positively homogeneous
allocation cost selects underactuated-like behavior on a fully actuated
multirotor. This cost class includes norms and power laws, including
standard minimum-norm pseudoinverse allocation. We use fiberwise
infimization to transport the actuator-output cost to wrench space,
yielding a homogeneous infimal allocation cost. Restriction to unit
zero-moment forces defines a directional landscape whose minimizing
directions are independent of force magnitude. We show that aligning
any fixed minimizing direction with the required world-frame force
minimizes instantaneous and integrated zero-moment allocation costs. The
resulting body force remains in a fixed one-dimensional subspace without
transverse components. We derive an exact optimality-gap decomposition
quantifying the additional cost of every other attitude trajectory and
characterizing the minimum-cost class. When a disturbance or
prescribed-motion change modifies the required-force direction, a fully
actuated platform can immediately generate the modified force using
transverse authority and reorient toward the new minimum-cost family. A
structurally underactuated platform must reorient before generating that
force. Thus, homogeneous allocation-cost minimization explains why
optimal nominal motion can be underactuated-like, while full actuation
remains valuable for disturbance rejection and transient force
generation.
}

\newcommand{\PaperKeywords}{%
Control allocation, emergent underactuation, fully actuated multirotors,
homogeneous allocation costs%
}

\title{\papertitle}

\ifSHORTVERSION
  \ifanonymous
    \author{Anonymous Author}
  \else
    \author{\PaperAuthorsAffiliated%
        \thanks{\AffiliationOne}%
        \thanks{\AffiliationTwo}%
        \thanks{\FundingStatement}%
    }
  \fi
\else
  \ifanonymous
    \author{Anonymous Author}
  \else
    \author{\PaperAuthorsAffiliated}
  \fi
\fi

\ifanonymous
  \hypersetup{
    pdftitle={\papertitle},
    pdfauthor={Anonymous Author}
  }
\else
  \hypersetup{
    pdftitle={\papertitle},
    pdfauthor={Antonio Franchi}
  }
\fi

\ifSHORTVERSION
  \makeatletter
  \@ifundefined{keywords}{%
    \newenvironment{keywords}
      {\par\smallskip\noindent\textbf{Keywords:}\ }
      {\par\smallskip}
  }{}
  \makeatother
\fi

\ifSHORTVERSION
\else
  \makeatletter
  \newcommand{\MakeNeutralTitle}{%
    \twocolumn[{%
      \begin{@twocolumnfalse}
        \vspace*{-1.5em}

        {\raggedright
         \fontsize{17}{20}\selectfont
         \bfseries
         \papertitle\par}

        \vspace{0.9em}

        \ifanonymous
          {\raggedright
           \normalsize
           \textbf{Anonymous Author}\par}
        \else
          {\raggedright
           \normalsize
           \textbf{\PaperAuthorsAffiliated}\par}

          \vspace{0.35em}

          {\raggedright
           \small
           \AffiliationOne\par
           \AffiliationTwo\par
           \CorrespondingAuthor\par
           \FundingStatement\par}
        \fi

        \vspace{1.1em}

        \noindent
        {\small\bfseries Abstract}\par

        \vspace{0.25em}

        {\small\noindent
         \PaperAbstract\par}

        \vspace{0.8em}

        {\small\noindent
         \textbf{Keywords:}
         \PaperKeywords.\par}

        \vspace{1.6em}
      \end{@twocolumnfalse}
    }]%
  }
  \makeatother
\fi

\newcommand{\RealNumbers}{\mathbb{R}}
\newcommand{\NonnegativeReals}{\mathbb{R}_{\geq 0}}
\newcommand{\PositiveReals}{\mathbb{R}_{>0}}
\newcommand{\RotationGroup}{\mathrm{SO}(3)}
\newcommand{\UnitSphere}{\mathbb{S}^{2}}
\newcommand{\ActuatorStateSpace}{\mathcal{V}}
\newcommand{\ActuatorOutputSpace}{\mathcal{X}}
\newcommand{\WrenchSpace}{\mathcal{W}}
\newcommand{\GenericCone}{\mathcal{C}}
\newcommand{\actuatorCount}{n}
\newcommand{\wrenchDimension}{m}
\newcommand{\fullWrenchDimension}{6}
\newcommand{\spatialDimension}{3}
\newcommand{\timeHorizon}{T}
\newcommand{\timeCoordinate}{t}
\newcommand{\homogeneityDegree}{q}
\newcommand{\positiveScale}{\lambda}
\newcommand{\normOrder}{p}
\newcommand{\zeroScalar}{0}

\newcommand{\zeroMomentEmbedding}{\iota_{0}}

\newcommand{\genericFunction}{g}
\newcommand{\genericArgument}{z}

\newcommand{\actuatorStateArray}{\bm{v}}
\newcommand{\actuatorStateComponent}[1]{v_{#1}}
\newcommand{\actuatorOutputArray}{\bm{x}}
\newcommand{\actuatorOutputComponent}[1]{x_{#1}}
\newcommand{\wrenchCoordinates}{\bm{w}}
\newcommand{\forceCoordinates}{\bm{F}}
\newcommand{\momentCoordinates}{\bm{\tau}}
\newcommand{\zeroForceCoordinates}{\bm{0}_{3}}
\newcommand{\zeroMomentCoordinates}{\bm{0}_{3}}
\newcommand{\onesArray}{\bm{1}_{n}}
\newcommand{\actuatorIndex}{i}

\newcommand{\actuatorMap}{q}
\newcommand{\allocationMap}{a}
\newcommand{\physicalTaskMap}{f}
\newcommand{\allocationMatrix}{A}
\newcommand{\forceMatrix}{A_{F}}
\newcommand{\momentMatrix}{A_{\tau}}
\newcommand{\matrixRank}{\operatorname{rank}}
\newcommand{\linearSpan}{\operatorname{span}}
\newcommand{\allocationFiber}{\mathcal{F}}
\newcommand{\fiberAt}[1]{\allocationFiber_{#1}}
\newcommand{\actuatorCost}{c}
\newcommand{\wrenchValue}{\Gamma}
\newcommand{\forceValue}{\gamma}
\newcommand{\directionalValue}{\bar{\gamma}}
\newcommand{\minimumDirectionalValue}{\gamma_{\star}}
\newcommand{\minimumDirectionSet}{\mathcal{U}_{\star}}
\newcommand{\minimumOutputSet}{\mathcal{X}_{\star}}
\newcommand{\infimum}{\operatorname*{inf}}
\newcommand{\minimum}{\operatorname*{min}}
\newcommand{\argminimum}{\operatorname*{arg\,min}}

\newcommand{\actuatorWeightMatrix}{W}
\newcommand{\actuatorCostWeight}[1]{\alpha_{#1}}
\newcommand{\actuatorPowerCoefficient}[1]{\kappa_{#1}}
\newcommand{\absoluteValue}[1]{\left\lvert #1\right\rvert}
\newcommand{\vectorNorm}[1]{\left\lVert #1\right\rVert}
\newcommand{\vectorNormOrder}[2]{\left\lVert #1\right\rVert_{#2}}
\newcommand{\transpose}{\mathsf{T}}

\newcommand{\WorldFrame}{\mathcal{F}_{W}}
\newcommand{\BodyFrame}{\mathcal{F}_{B}}
\newcommand{\attitude}{R}
\newcommand{\optimalAttitude}{R_{\star}}
\newcommand{\optimalAttitudeSet}{\mathcal{R}_{\star}}
\newcommand{\positionTrajectory}{\bm{p}}
\newcommand{\vehicleMass}{m_{v}}
\newcommand{\gravityAcceleration}{g}
\newcommand{\worldVertical}{\bm{e}_{3}^{W}}
\newcommand{\worldVerticalCoordinates}{\begin{bmatrix}0&0&1\end{bmatrix}^{\transpose}}
\newcommand{\worldAxisOne}{\bm{e}_{1}^{W}}
\newcommand{\bodyAxisOne}{\bm{e}_{1}^{B}}
\newcommand{\bodyAxisThree}{\bm{e}_{3}^{B}}
\newcommand{\requiredWorldForce}{\bm{F}_{d}^{W}}
\newcommand{\requiredBodyForce}{\bm{F}_{d}^{B}}
\newcommand{\generatedWorldForce}{\bm{F}_{p}^{W}}
\newcommand{\requiredForceMagnitude}{r}
\newcommand{\requiredForceDirection}{\bm{n}}
\newcommand{\unitForceDirection}{\bm{u}}
\newcommand{\minimumForceDirection}{\bm{a}_{\star}}
\newcommand{\forceProjectorComplement}{I_{3}-\minimumForceDirection\minimumForceDirection^{\transpose}}
\newcommand{\horizontalAttitude}{R_{\mathrm{H}}}

\newcommand{\trajectoryCostOf}[2]{J[#1;#2]}
\newcommand{\minimumTrajectoryCostOf}[1]{J_{\star}[#1]}
\newcommand{\instantaneousGapOf}[1]{\delta_{#1}}
\newcommand{\accumulatedGapOf}[2]{\Delta J[#1;#2]}
\newcommand{\differentialTime}{\mathrm{d}t}
\newcommand{\AttitudeTrajectorySpace}{\mathcal{R}_{\positionTrajectory}}
\newcommand{\attitudeOne}{R_{1}}
\newcommand{\attitudeTwo}{R_{2}}
\newcommand{\attitudePreorder}{\preceq_{\positionTrajectory}}
\newcommand{\attitudeCostEquivalence}{\sim_{\positionTrajectory}}
\newcommand{\OptimalAttitudeTrajectorySet}{\mathcal{R}_{\star,\positionTrajectory}}
\newcommand{\AttitudeQuotientSpace}{\AttitudeTrajectorySpace/\attitudeCostEquivalence}
\newcommand{\relativeGapOf}[2]{\eta[#1;#2]}

\newcommand{\attitudeClassOrder}{\leq_{\mathrm{c}}}

\newcommand{\horizontalPosition}{s}
\newcommand{\constantAltitude}{h}
\newcommand{\horizontalAcceleration}{\ddot{s}}

\newcommand{\externalWorldForce}{\bm{F}_{e}^{W}}
\newcommand{\disturbedRequiredWorldForce}{\bm{F}_{d,e}^{W}}
\newcommand{\disturbedForceMagnitude}{r_{e}}
\newcommand{\disturbedForceDirection}{\bm{n}_{e}}
\newcommand{\disturbedOptimalAttitudeSet}{\mathcal{R}_{\star,e}}
\newcommand{\fixedAttitudeGap}{\delta_{e,R}}

\begin{document}

\ifSHORTVERSION

  \maketitle

  \begin{abstract}
  \PaperAbstract
  \end{abstract}


\else

  \pagestyle{plain}

  \bstctlcite{BSTcontrol}

  \MakeNeutralTitle

\fi

\section{Introduction}
\label{sec:introduction}

Conventional parallel-axis multirotors generate force along one
body-fixed direction and must reorient it to accelerate laterally.
Fully actuated multirotors remove this structural coupling by generating
a complete body wrench independently of attitude, enabling independent
pose control and physical interaction
\cite{Ryll2015TCST,Rashad2020Review}. Representative architectures use
bidirectional spatial thrust arrangements
\cite{Brescianini2016ICRA,Park2018TMECH}, unidirectional tilted
propellers \cite{Tognon2018RAL}, active tilting mechanisms
\cite{Kamel2018RAM,Allenspach2020IJRR}, or reconfigurable structures
\cite{Zhao2018RAL,Aboudorra2024JINT}. This actuation freedom is
particularly valuable in aerial manipulation and contact inspection,
where force application cannot generally be subordinated to attitude
motion \cite{Ruggiero2018RAL,SanchezCuevas2020Sensors}.

In these platforms the control-allocation layer distributes a requested wrench among the
available actuators. Redundant actuation supports constraints, failure
accommodation, and secondary optimization
\cite{bodson2002,johansen2013}. For multirotors, allocation geometry
determines force--moment coupling and failure tolerance
\cite{Michieletto2018TRO}, while optimal allocation for prescribed
wrenches has also been considered \cite{Dyer2019ICRA}. Standard
homogeneous costs include norms and power laws, notably the
minimum-norm cost underlying pseudoinverse allocation.

A different redundancy problem arises when position is prescribed and attitude
remains free. Position determines the required world-frame force,
whereas attitude selects its body-frame direction. The directionally
nonuniform actuation of the platform therefore assigns different infimal allocation
costs to attitudes realizing the same translation.

This work shows that  minimizing a positively homogeneous allocation cost  selects
underactuated-like behavior on a fully actuated multirotor. Aligning any
fixed minimizing direction with the required world-frame force confines
the body force to a fixed one-dimensional subspace without transverse
components. Full actuation remains valuable when disturbances or motion
changes modify the required-force direction: transverse authority
generates the new force immediately while reorientation recovers
minimum-cost alignment. Thus, nominal efficiency favors
underactuated-like motion, while full actuation supports tracking
resilience.

The main contributions of this study are:
\begin{enumerate}[leftmargin=*]
\item a fiberwise infimal-cost formulation showing that positively homogeneous
allocation costs induce force-magnitude-independent minimizing
directions;

\item a characterization of the minimum-cost attitude family, emergent
underactuation, and the exact optimality-gap ordering of full-pose
trajectories realizing the same position motion; and

\item a parallel-thrust lower bound and disturbance-response
interpretation clarifying the efficiency and tracking roles of
underactuated and fully actuated force generation.
\end{enumerate}
\ifSHORTVERSION
\else
The logical structure of this work is summarized in
\cref{fig:results-dependency-map}. The map separates the allocation and
homogeneity foundations from the core magnitude--direction result, its
underactuated-like consequence, the full-pose ordering, reserve wrench
authority, and the numerical illustration.
\begin{figure*}[t!]
\centering
\resizebox{0.9\textwidth}{!}{%
\begin{tikzpicture}[
  font=\sffamily,
  node distance=7mm and 13mm,
  formal/.style={
    rounded corners=2pt,
    draw=black!68,
    line width=0.65pt,
    fill=white,
    align=center,
    inner xsep=7pt,
    inner ysep=6pt,
    text width=8.4cm
  },
  core/.style={
    formal,
    draw=blue!72!black,
    line width=1.15pt,
    fill=blue!4,
    text width=11.15cm
  },
  consequence/.style={
    formal,
    draw=orange!78!black,
    fill=orange!4,
    text width=4.85cm
  },
  evidence/.style={
    formal,
    draw=green!55!black,
    fill=green!4,
    text width=11.3cm
  },
  dependency/.style={
    -{Latex[length=2.5mm,width=1.7mm]},
    line width=0.8pt,
    draw=black!72
  },
  interpretation/.style={
    -{Latex[length=2.5mm,width=1.7mm]},
    line width=0.75pt,
    dashed,
    draw=black!55
  },
  group/.style={
    rounded corners=4pt,
    draw=black!28,
    line width=0.55pt,
    fill=black!1,
    inner sep=8pt
  },
  grouptitle/.style={
    font=\bfseries\small,
    text=black!66,
    anchor=south west
  }
]

\node[formal] (fibers) {
  {\bfseries\normalsize Allocation geometry}\par\smallskip
  {\normalsize Actuator outputs that generate the same wrench form an
  allocation fiber. Full actuation makes every considered fiber
  nonempty; redundant actuation can give a fiber multiple points.}\par\smallskip
  {\itshape\small\color{black!72} \cref{def:full-redundant-actuation,def:allocation-fiber,def:infimal-allocation-value,def:zero-moment-allocation-value}}
};

\node[formal, right=of fibers] (homogeneity) {
  {\bfseries\normalsize Homogeneous fiberwise cost}\par\smallskip
  {\normalsize Minimizing a positively homogeneous actuator cost on every
  fiber produces homogeneous wrench and infimal zero-moment allocation costs.}\par\smallskip
  {\itshape\small\color{black!72} \cref{def:positive-homogeneity,prop:inheritance-homogeneity}}
};

\node[core, below=14mm of $(fibers.south)!0.5!(homogeneity.south)$]
  (direction) {
  {\bfseries\normalsize Magnitude and direction separate}\par\smallskip
  {\normalsize Changing the required-force magnitude scales the minimum
  allocation cost but does not change the minimizing body-force
  directions.}\par\smallskip
  {\itshape\small\color{black!72} \cref{def:directional-allocation-cost,prop:magnitude-direction-separation}}
};

\node[core, below=of direction] (emergence) {
  {\bfseries\normalsize Emergent underactuation}\par\smallskip
  {\normalsize A minimum-cost attitude aligns a fixed minimum-allocation
  direction with the required world-frame force. Consequently,}\par\smallskip
  {$\displaystyle
    \requiredBodyForce(\timeCoordinate)
=\requiredForceMagnitude(\timeCoordinate)\minimumForceDirection,
    \quad
    \forceProjectorComplement    \requiredBodyForce(\timeCoordinate)=\zeroForceCoordinates.$}\par\smallskip
  {\normalsize The platform remains fully actuated, but its minimum-cost
  translational force has the structure of a one-directional
  underactuated platform.}\par\smallskip
  {\itshape\small\color{black!72} \cref{thm:minimum-cost-attitude-family,cor:emergent-underactuation}}
};

\node[consequence, below left=15mm and 9mm of emergence, xshift=25mm]
  (trajectory) {
  {\bfseries\normalsize Trajectory-level optimality}\par\smallskip
  {\normalsize Every attitude trajectory has a nonnegative accumulated
  allocation gap:}\par\smallskip
  {$\displaystyle
    \trajectoryCostOf{\attitude}{\positionTrajectory}
    =\minimumTrajectoryCostOf{\positionTrajectory}
    +\accumulatedGapOf{\attitude}{\positionTrajectory}.$}\par\smallskip
  {\normalsize The minimum class consists exactly of trajectories remaining
  in the minimum-direction set almost everywhere.}\par\smallskip
  {\itshape\small\color{black!72} \cref{def:allocation-optimality-gap,prop:exact-cost-decomposition,prop:cost-ordering,prop:minimum-equivalence-class}}
};

\node[consequence, below=15mm of emergence] (parallel) {
  {\bfseries\normalsize Parallel-thrust lower bound}\par\smallskip
  {\normalsize With summed nonnegative actuator output, parallel and
  codirectional rotor forces avoid cancellation and attain the universal
  force-generation lower bound whenever zero moment is feasible.}\par\smallskip
  {\itshape\small\color{black!72} \cref{rem:parallel-thrust-lower-bound}}
};

\node[consequence, below right=15mm and 9mm of emergence,xshift=-25mm]
  (transient) {
  {\bfseries\normalsize Full actuation during transitions}\par\smallskip
  {\normalsize When the required-force direction changes, full actuation can
  generate transverse force immediately and later recover minimum-cost
  alignment. A structurally underactuated platform must first
  reorient.}\par\smallskip
  {\itshape\small\color{black!72} \cref{sec:resilient-full-actuation,eq:fixed-attitude-disturbance-gap,eq:disturbed-cost-decomposition}}
};

\node[evidence, below=29mm of parallel] (numerics) {
  {\bfseries\normalsize Numerical illustration}\par\smallskip
  {\normalsize UA preserves the one-directional force structure but develops
  tracking error during reorientation. FA-opt tracks immediately and
  approaches the new minimum-cost attitude. FA-const also tracks
  immediately but retains a larger allocation cost.}\par\smallskip
  {\itshape\small\color{black!72} \cref{sec:numerical-illustration,fig:numerical-comparison}}
};

\coordinate (fibersStart) at
  ($(fibers.south)!0.5!(fibers.south east)$);

\coordinate (homogeneityStart) at
  ($(homogeneity.south)!0.5!(homogeneity.south west)$);

\draw[dependency]
  (fibersStart) --
  (fibersStart |- direction.north);

\draw[dependency]
  (homogeneityStart) --
  (homogeneityStart |- direction.north);

\draw[dependency] (direction) -- (emergence);
\draw[dependency] (emergence) -- (trajectory);

\coordinate (fibersStartB) at
  ($(fibers.south)!0.9!(fibers.south west)$);

\coordinate (homogeneityStartB) at
  ($(homogeneity.south)!0.9!(homogeneity.south east)$);

\draw[dependency] (fibersStartB) to[out=-90,in=180] (parallel.north west);
\draw[dependency] (homogeneityStartB) to[out=-90,in=0] (parallel.north east);

\draw[dependency] (emergence) -- (transient);

\draw[interpretation] (emergence) -- (parallel);
\draw[interpretation] (trajectory.south) -- (numerics.north west);
\draw[interpretation] (parallel.south) -- (numerics.north);
\draw[interpretation] (transient.south) -- (numerics.north east);

\begin{scope}[on background layer]
  \node[group, fit=(fibers)(homogeneity)] (foundationgroup) {};
  \node[group, draw=blue!28, fill=blue!1,
        fit=(direction)(emergence)] (coregroup) {};
  \node[group, draw=orange!35, fill=orange!1,
        fit=(trajectory)(parallel)(transient)] (consequencegroup) {};
  \node[group, draw=green!32!black, fill=green!1,
        fit=(numerics)] (evidencegroup) {};
\end{scope}

\node[grouptitle] at ($(foundationgroup.north west)+(0,1.5mm)$)
  {FOUNDATIONS};
\node[grouptitle, text=blue!60!black]
  at ($(coregroup.north west)+(0,1.5mm)$) {CORE THEORY};
\node[grouptitle, text=orange!65!black]
  at ($(consequencegroup.north west)+(0,1.5mm)$) {CONSEQUENCES};
\node[grouptitle, text=green!42!black]
  at ($(evidencegroup.north west)+(0,1.5mm)$) {ILLUSTRATION};

\node[anchor=north east, font=\normalsize, align=left]
  at ($(evidencegroup.south east)+(0,-2.5mm)$) {
  \tikz[baseline=-0.55ex]\draw[dependency] (0,0)--(8mm,0);
  formal dependence
  \quad
  \tikz[baseline=-0.55ex]\draw[interpretation] (0,0)--(8mm,0);
  interpretation or evidence
};

\end{tikzpicture}%
}
\caption{Dependency map of the main results. Each box states the
control-theoretic content in accessible terms and identifies the formal result
that establishes it. Solid arrows denote formal dependence; dashed arrows denote
architectural interpretation or numerical evidence.}
\label{fig:results-dependency-map}
\end{figure*}
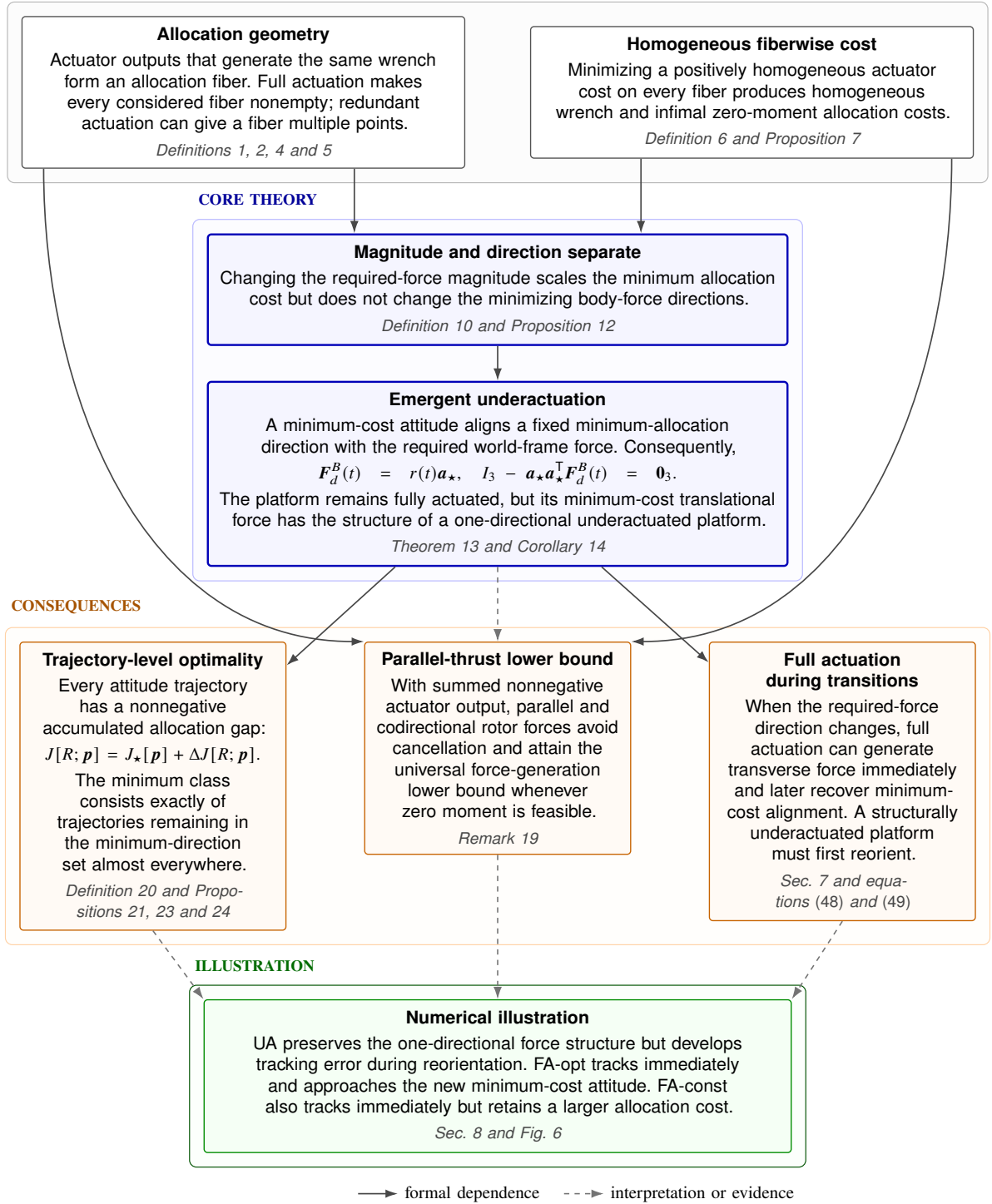

\fi

\ifSHORTVERSION
\else

\section{Related Work}
\label{sec:related-work}

Multirotor modeling and control have developed from platform-specific
dynamic models and near-hover stabilization
\cite{Pounds2010CEP,Mahony2012RAM} toward nonlinear and geometric
control on $\RotationGroup$ and $\mathrm{SE}(3)$
\cite{Hua2009TAC,Lee2010CDC}. Differential flatness and
polynomial trajectory generation have enabled dynamically feasible and
aggressive flight \cite{Mellinger2011ICRA,Mellinger2012IJRR}, while
efficient motion primitives and aerodynamic feedforward have extended
these capabilities to online planning and high-speed tracking
\cite{Mueller2015TRO,Faessler2018RAL}. These approaches exploit the
underactuated structure of conventional parallel-thrust multirotors, for
which translation determines a preferred thrust direction and attitude
orients that direction in space.

Aerospace flight-control research emphasizes robustness, control
effectiveness, handling qualities, and integrated vehicle--controller
design. Incremental nonlinear dynamic inversion provides high-performance
attitude control with online adaptation of control effectiveness
\cite{Smeur2016JGCD}, and its combination with structured robust control
supports disturbance rejection and full-pose tracking of redundantly actuated
multirotors \cite{Hachem2026JGCD}. Multirotor co-design
has addressed robust performance, handling qualities, and motor failures
jointly \cite{Mabboux2024AST}, whereas bidirectional fixed-pitch
propulsion has been studied through unified modeling, nonlinear control,
and control allocation \cite{Maier2018IROS}. These developments show how
actuation design and feedback control jointly determine the usable
wrench authority of a multirotor.

A complementary research line changes the actuation geometry itself.
Tilted, vectored, bidirectional, and reconfigurable rotor arrangements
provide complete or configuration-dependent wrench authority
\cite{Ryll2015TCST,Rajappa2015ICRA,Brescianini2016ICRA,
Tognon2018RAL,Park2018TMECH,Kamel2018RAM,Yao2018IFAC,
Allenspach2020IJRR,Kotarski2021AS,Zhao2018RAL,Aboudorra2024JINT}.
Allocation-based classifications clarify how rotor number, placement,
orientation, and tilting mechanisms determine attainable wrenches and
platform capabilities \cite{Michieletto2018TRO,Hamandi2021IJRR,
Rashad2020Review}. At the control level, intrinsic tracking on
homogeneous Riemannian manifolds provides a unified treatment of fully
actuated mechanical systems that includes omnidirectional aerial robots
\cite{Welde2024LCS}. This literature mainly exploits full actuation for
independent pose tracking, enhanced wrench generation, or
configuration-specific objectives. This work instead investigates which
attitude is selected when only translation is prescribed and allocation
cost determines the use of the remaining actuation freedom.

Redundant actuation and fault tolerance provide a further connection
between vehicle geometry and allocation. General control-allocation
methods distribute commanded inputs while accommodating actuator
constraints and secondary objectives \cite{bodson2002,johansen2013}.
For multirotors, allocation has been studied through force--moment
coupling, control authority, failure tolerance, and energy-aware wrench
realization \cite{Michieletto2018TRO,Dyer2019ICRA,santos2022}.
Relaxed-hover solutions show how attitude and residual rotor authority
can sustain flight after severe actuator loss
\cite{Mueller2016IJRR}, while controllability analyses characterize the
authority retained under rotor degradation and failure
\cite{Du2015JGCD}. These studies assess or optimize allocation for a
specified wrench, failure condition, or operating regime. The
redundancy considered here additionally includes the attitude freedom
that maps one prescribed world-frame force into different body-frame
allocation demands.

This distinction is especially relevant to aerial manipulation and
physical interaction. Early aerial manipulators combined multirotors
with robotic mechanisms and interaction controllers
\cite{Keemink2012ICRA,Fumagalli2014RAM}. Compliance and impact-absorbing
mechanisms subsequently improved contact robustness
\cite{Bartelds2016RAL}, while conventional underactuated platforms were
shown capable of applying sustained forces through coordinated attitude
and contact control \cite{Wopereis2017ICRA}. Fully actuated platforms
permit contact force and vehicle orientation to be regulated more
independently, as demonstrated for omnidirectional contact inspection
\cite{Bodie2019RSS,SanchezCuevas2020Sensors}. Reviews of aerial
manipulation and multirotor physical interaction document the resulting
diversity of platforms, mechanisms, controllers, and tasks
\cite{Ruggiero2018RAL,Ladig2021JRM,Xu2023Survey}.

Across these research directions, actuation geometry is commonly
evaluated through wrench feasibility, tracking performance, robustness,
or the cost of realizing a prescribed wrench. The question considered
here is different: for a prescribed position motion and free attitude,
which full-pose realization minimizes allocation cost? The fiberwise
formulation shows that positive homogeneity makes the minimizing
directions independent of force magnitude and can select a fixed
body-force direction, thereby recovering an underactuated-like nominal
realization from a fully actuated architecture.

\fi

\section{Actuation Geometry and Fiberwise Cost}
\label{sec:actuation-geometry}

Consider a multirotor equipped with \(\actuatorCount\) propellers having fixed
positions and orientations relative to the body frame. The forces and moments
generated by the propellers yield the standard allocation model, whose
composition with the signed-quadratic propeller model gives the
propeller-speed-to-wrench map \cite{Michieletto2018TRO}.
Let \(\ActuatorStateSpace\subseteq\RealNumbers^{\actuatorCount}\) be the
admissible propeller-speed space,
\(\ActuatorOutputSpace\subseteq\RealNumbers^{\actuatorCount}\) the
\emph{actuator-output} space, and
\(\WrenchSpace=\RealNumbers^{\wrenchDimension}\) the selected body-wrench
coordinate space. The corresponding coordinate arrays are
\(\actuatorStateArray\in\ActuatorStateSpace\),
\(\actuatorOutputArray\in\ActuatorOutputSpace\), and
\(\wrenchCoordinates\in\WrenchSpace\), respectively.
After absorbing the positive thrust coefficients into \(\actuatorStateArray\), the propeller-speed-to-wrench map factors as
\begin{equation}
 \wrenchCoordinates
 =\physicalTaskMap(\actuatorStateArray)
 =\allocationMap\!\left(\actuatorMap(\actuatorStateArray)\right),
 \label{eq:physical-task-map}
\end{equation}
where the componentwise signed-quadratic actuator map is
\begin{equation}
 \actuatorOutputArray=\actuatorMap(\actuatorStateArray),
 \quad
 \actuatorOutputComponent{\actuatorIndex}
 =\actuatorStateComponent{\actuatorIndex}
 \absoluteValue{\actuatorStateComponent{\actuatorIndex}},
 \label{eq:actuator-map}
\end{equation}
and \(\allocationMap:\ActuatorOutputSpace\to\WrenchSpace\) is the linear
allocation map. For unidirectional propellers, the admissible state and output
spaces are restricted to nonnegative coordinates. For bidirectional
propellers, signed coordinates describe the two possible force directions along
each fixed propeller axis.

The linear allocation map is represented in coordinates by
\begin{equation}
 \wrenchCoordinates
 =\allocationMap(\actuatorOutputArray)
 =\allocationMatrix\actuatorOutputArray,
 \quad
 \allocationMatrix\in
 \RealNumbers^{\wrenchDimension\times\actuatorCount}.
 \label{eq:allocation-map}
\end{equation}
This work concerns the geometry and cost of the
linear allocation stage; propeller-speed nonlinearity, regularity and actuator dynamics are outside the present analysis.

\ifSHORTVERSION
\else
Before developing the individual definitions, \cref{fig:mathematical-allocation-synopsis}
provides a compact roadmap of the complete construction. It shows how the
physical maps lead to allocation fibers and their minimum outputs, and how the
zero-moment and unit-sphere restrictions ultimately define the minimizing
force directions studied in \cref{sec:minimum-directions}.
\begin{figure}[t!]
\centering
\begin{minipage}{0.99\columnwidth}
\centering

{\bfseries Physical factorization}
\[
\begin{array}{ccccc}
 \ActuatorStateSpace
 & \xrightarrow{\ \actuatorMap\ }
 & \ActuatorOutputSpace
 & \overset{\allocationMap}{\twoheadrightarrow}
 & \WrenchSpace
 \\
 \actuatorStateArray
 & \longmapsto
 & \actuatorOutputArray
 & \longmapsto
 & \wrenchCoordinates
\end{array}
\quad
\physicalTaskMap=\allocationMap\circ\actuatorMap,
\quad
\allocationMap(\ActuatorOutputSpace)=\WrenchSpace.
\]

\vspace{0.6em}
{\bfseries Fibers and fiberwise infimization}
\[
\minimumOutputSet(\wrenchCoordinates)
\subseteq
\fiberAt{\wrenchCoordinates}
=
\allocationMap^{-1}(\{\wrenchCoordinates\})
\subseteq
\ActuatorOutputSpace,
\]
\[
\actuatorCost:\ActuatorOutputSpace\to\NonnegativeReals,
\quad
\wrenchValue:\WrenchSpace\to\NonnegativeReals,
\quad
\minimumOutputSet:\WrenchSpace\rightrightarrows\ActuatorOutputSpace,
\]
\[
\wrenchValue(\wrenchCoordinates)
=
\infimum_{\actuatorOutputArray\in\fiberAt{\wrenchCoordinates}}
\actuatorCost(\actuatorOutputArray),
\quad
\actuatorCost(\actuatorOutputArray)
\geq
\wrenchValue(\allocationMap(\actuatorOutputArray)),
\]
\[
\actuatorCost(\actuatorOutputArray)
=
\wrenchValue(\allocationMap(\actuatorOutputArray))
\quad\Longleftrightarrow\quad
\actuatorOutputArray
\in
\minimumOutputSet(\allocationMap(\actuatorOutputArray)).
\]

\vspace{0.6em}
{\bfseries Zero-moment embedding and directional restriction}
\[
\begin{array}{ccccc}
 \UnitSphere
 & \hookrightarrow
 & \RealNumbers^{\spatialDimension}
 & \xrightarrow{\ \zeroMomentEmbedding\ }
 & \WrenchSpace
 \\
 \Big\downarrow\scriptstyle{\directionalValue}
 &&
 \Big\downarrow\scriptstyle{\forceValue}
 &&
 \Big\downarrow\scriptstyle{\wrenchValue}
 \\
 \NonnegativeReals
 & =
 & \NonnegativeReals
 & =
 & \NonnegativeReals
\end{array}
\]
\[
\forceValue=\wrenchValue\circ\zeroMomentEmbedding,
\quad
\directionalValue=\left.\forceValue\right|_{\UnitSphere},
\quad
\zeroMomentEmbedding(\forceCoordinates)
=
\begin{bmatrix}
 \forceCoordinates\\
 \zeroMomentCoordinates
\end{bmatrix}.
\]

\vspace{0.6em}
{\bfseries Minimum directional cost and minimizing directions}
\[
\minimumDirectionalValue
=
\minimum_{\unitForceDirection\in\UnitSphere}
\directionalValue(\unitForceDirection),
\quad
\minimumDirectionSet
=
\directionalValue^{-1}(\{\minimumDirectionalValue\})
=
\argminimum_{\unitForceDirection\in\UnitSphere}
\directionalValue(\unitForceDirection),
\]
\[
\minimumDirectionSet\times\{\zeroMomentCoordinates\}
\subseteq
\UnitSphere\times\{\zeroMomentCoordinates\}
\subseteq
\RealNumbers^{\spatialDimension}\times\{\zeroMomentCoordinates\}
\subseteq
\WrenchSpace.
\]

\end{minipage}
\caption{Synopsis of the actuation and allocation construction.
The physical task map factors through the actuator and allocation maps;
fiberwise infimization defines the infimal allocation cost and minimum-output
map; and the zero-moment embedding followed by restriction to unit force
directions defines the directional infimal allocation cost and its minimizing
directions.}
\label{fig:mathematical-allocation-synopsis}
\end{figure}

\fi

\begin{definition}[Full actuation and redundant actuation]
\label{def:full-redundant-actuation}
The platform is fully actuated on $\WrenchSpace$ if every
$\wrenchCoordinates\in\WrenchSpace$ admits an
$\actuatorOutputArray\in\ActuatorOutputSpace$ satisfying
$\allocationMap(\actuatorOutputArray)=\wrenchCoordinates$.
The platform has redundant actuation if it is fully actuated and there
exist distinct
$\actuatorOutputArray_{1},\actuatorOutputArray_{2}
\in\ActuatorOutputSpace$ such that
$\allocationMap(\actuatorOutputArray_{1})
=\allocationMap(\actuatorOutputArray_{2})$.
If
$\ActuatorOutputSpace=\RealNumbers^{\actuatorCount}$, full actuation is
equivalent to
$\matrixRank\allocationMatrix=\wrenchDimension$; in this case,
$\actuatorCount>\wrenchDimension$ implies redundant actuation. If
$\ActuatorOutputSpace$ is sign constrained, full row rank remains
necessary but is not sufficient for full actuation, which additionally
requires
$\allocationMap(\ActuatorOutputSpace)=\WrenchSpace$.
\end{definition}

\begin{definition}[Allocation fiber]
\label{def:allocation-fiber}
The \emph{fiber} of $\wrenchCoordinates\in\WrenchSpace$ is
\begin{equation}
 \fiberAt{\wrenchCoordinates}
 :=
 \left\{
 \actuatorOutputArray\in\ActuatorOutputSpace:
 \allocationMap(\actuatorOutputArray)=\wrenchCoordinates
 \right\}.
 \label{eq:allocation-fiber}
\end{equation}
\end{definition}

The fiber contains all  actuator-output arrays realizing the
same wrench. Under redundant actuation, a fiber can contain multiple
points and therefore task-preserving internal freedom.

\begin{assumption}
\label{ass:standing-actuation}
The selected wrench space is the complete body-wrench coordinate space,
so that
$\WrenchSpace=\RealNumbers^{\fullWrenchDimension}$
and
$\wrenchDimension=\fullWrenchDimension$.
The platform is fully actuated on $\WrenchSpace$.
Consequently, every allocation fiber is nonempty.
\end{assumption}

\subsection{Minimum cost on the fibers}

Consider an \emph{actuator-output cost}
$\actuatorCost:\ActuatorOutputSpace\to\NonnegativeReals$ and the fiberwise infimization  defined in the following.

\begin{definition}[Infimal allocation cost and minimum-output map]
\label{def:infimal-allocation-value}
The infimal allocation cost is the function
\begin{equation}
 \wrenchValue:
 \WrenchSpace
 \to
 \NonnegativeReals,
 \quad
 \wrenchValue(\wrenchCoordinates)
 :=
 \infimum_{\actuatorOutputArray\in\fiberAt{\wrenchCoordinates}}
 \actuatorCost(\actuatorOutputArray).
 \label{eq:wrench-value}
\end{equation}
The minimum-output map is the set-valued map
\begin{equation}
 \minimumOutputSet:
 \WrenchSpace
 \rightrightarrows
 \ActuatorOutputSpace,
 \quad
 \minimumOutputSet(\wrenchCoordinates)
 :=
 \left\{
 \actuatorOutputArray\in\fiberAt{\wrenchCoordinates}:
 \actuatorCost(\actuatorOutputArray)
 =
 \wrenchValue(\wrenchCoordinates)
 \right\}.
 \label{eq:minimum-output-set}
\end{equation}
\end{definition}

The set $\minimumOutputSet(\wrenchCoordinates)$ is nonempty if and only
if the infimum in \cref{eq:wrench-value} is attained. In that case,
\begin{equation}
 \minimumOutputSet(\wrenchCoordinates)
 =
 \argminimum_{\actuatorOutputArray\in\fiberAt{\wrenchCoordinates}}
 \actuatorCost(\actuatorOutputArray) \subseteq \fiberAt{\wrenchCoordinates}.
 \label{eq:minimum-output-set-argmin}
\end{equation}

\ifSHORTVERSION
\else
The geometric role of fiberwise infimization is detailed in
\cref{fig:fiberwise-allocation-construction}. Fixing a wrench selects its
allocation fiber in $\ActuatorOutputSpace$; minimizing $\actuatorCost$ on that
fiber produces both the value $\wrenchValue(\wrenchCoordinates)$ and the
possibly set-valued collection $\minimumOutputSet(\wrenchCoordinates)$ of
outputs attaining it.
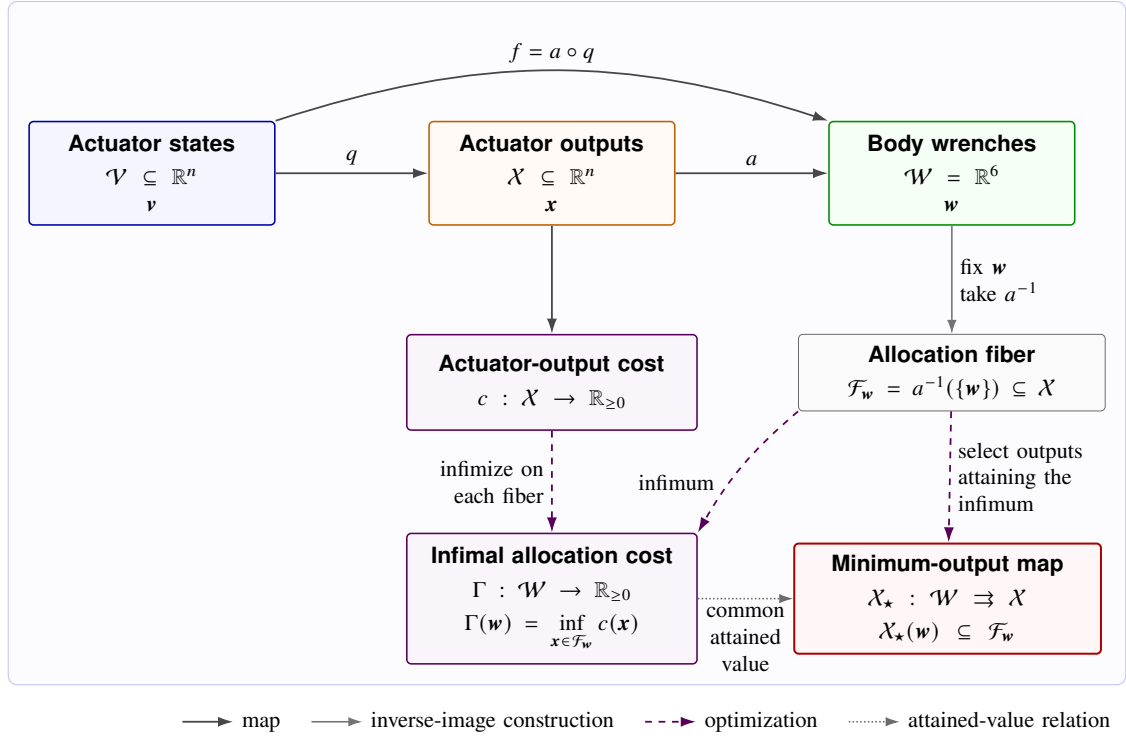
\begin{figure*}[t!]
\centering
\resizebox{0.85\textwidth}{!}{%
\begin{tikzpicture}[
  font=\sffamily,
  >=Latex,
  space/.style={
    rounded corners=2pt,
    draw=black!68,
    line width=0.7pt,
    fill=white,
    align=center,
    inner xsep=8pt,
    inner ysep=6pt,
    minimum height=15mm
  },
  state/.style={
    space,
    draw=blue!65!black,
    fill=blue!4,
    text width=3.3cm
  },
  output/.style={
    space,
    draw=orange!75!black,
    fill=orange!4,
    text width=3.3cm
  },
  wrench/.style={
    space,
    draw=green!55!black,
    fill=green!4,
    text width=3.3cm
  },
  cost/.style={
    space,
    draw=violet!70!black,
    fill=violet!4,
    text width=4.0cm
  },
  subset/.style={
    rounded corners=2pt,
    draw=black!58,
    fill=black!2,
    align=center,
    inner xsep=6pt,
    inner ysep=5pt,
    text width=4.4cm
  },
  selected/.style={
    subset,
    draw=red!68!black,
    fill=red!3,
    line width=0.9pt
  },
  map/.style={
    -{Latex[length=2.4mm,width=1.6mm]},
    line width=0.8pt,
    draw=black!72
  },
  inclusion/.style={
    -{Latex[length=2.2mm,width=1.5mm]},
    line width=0.7pt,
    draw=black!55
  },
  induced/.style={
    -{Latex[length=2.4mm,width=1.6mm]},
    line width=0.8pt,
    dashed,
    draw=violet!70!black
  },
  identify/.style={
    -{Latex[length=2.2mm,width=1.5mm]},
    line width=0.7pt,
    densely dotted,
    draw=black!58
  },
  group/.style={
    rounded corners=4pt,
    draw=blue!25,
    fill=blue!1,
    inner sep=9pt
  }
]

\node[state] (V) {
  {\bfseries\normalsize Actuator states}\par\smallskip
  $\ActuatorStateSpace\subseteq\RealNumbers^{\actuatorCount}$\par
  $\actuatorStateArray$
};

\node[output, right=24mm of V] (X) {
  {\bfseries\normalsize Actuator outputs}\par\smallskip
  $\ActuatorOutputSpace\subseteq\RealNumbers^{\actuatorCount}$\par
  $\actuatorOutputArray$
};

\node[wrench, right=24mm of X] (W) {
  {\bfseries\normalsize Body wrenches}\par\smallskip
  $\WrenchSpace=\RealNumbers^{\fullWrenchDimension}$\par
  $\wrenchCoordinates$
};

\draw[map] (V) -- node[above,font=\normalsize] {$\actuatorMap$} (X);
\draw[map] (X) -- node[above,font=\normalsize] {$\allocationMap$} (W);
\draw[map]
  (V.north east)
  to[bend left=18]
  node[above,font=\normalsize] (compositionLabel)
  {$\physicalTaskMap=\allocationMap\circ\actuatorMap$}
  (W.north west);

\coordinate (compositionTop) at
  ($(compositionLabel.north)+(0,2mm)$);

\node[cost, below=17mm of X] (c) {
  {\bfseries\normalsize Actuator-output cost}\par\smallskip
  $\actuatorCost:\ActuatorOutputSpace\to\NonnegativeReals$
};
\draw[map] (X) -- (c);

\node[subset, below=17mm of W] (fiber) {
  {\bfseries\normalsize Allocation fiber}\par\smallskip
  $\fiberAt{\wrenchCoordinates}
   =\allocationMap^{-1}(\{\wrenchCoordinates\})
   \subseteq\ActuatorOutputSpace$
};
\draw[inclusion] (W) -- node[right,font=\normalsize,align=left]
  {fix $\wrenchCoordinates$\\take $\allocationMap^{-1}$} (fiber);

\node[cost, below=16mm of c] (Gamma) {
  {\bfseries\normalsize Infimal allocation cost}\par\smallskip
  $\wrenchValue:\WrenchSpace\to\NonnegativeReals$\par\smallskip
  $\displaystyle
   \wrenchValue(\wrenchCoordinates)
   =\infimum_{\actuatorOutputArray\in\fiberAt{\wrenchCoordinates}}
    \actuatorCost(\actuatorOutputArray)$
};

\node[selected, right=15mm of Gamma] (Xstar) {
  {\bfseries\normalsize Minimum-output map}\par\smallskip
  $\minimumOutputSet:
   \WrenchSpace\rightrightarrows\ActuatorOutputSpace$\par\smallskip
  $\minimumOutputSet(\wrenchCoordinates)
   \subseteq\fiberAt{\wrenchCoordinates}$
};

\draw[induced] (c) -- node[left,font=\normalsize,align=right]
  {infimize on\\each fiber} (Gamma);
\draw[induced] (fiber.south west) to[bend right=10]
  node[below left,xshift=-3mm,font=\normalsize] {infimum} (Gamma.north east);
\draw[induced] (fiber) -- node[right,font=\normalsize,align=left]
  {select outputs\\attaining the\\ infimum} (Xstar);
\draw[identify] (Gamma) -- node[below,font=\normalsize,align=center]
  {common\\ attained\\ value} (Xstar);

\begin{scope}[on background layer]
  \node[group,
    fit=(V)(X)(W)(c)(fiber)(Gamma)(Xstar)(compositionTop)]
    (physicalgroup) {};
\end{scope}

\node[anchor=north east,font=\normalsize,align=left]
  at ($(physicalgroup.south east)+(0,-3mm)$) {
  \tikz[baseline=-0.55ex]\draw[map] (0,0)--(8mm,0);
  map
  \quad
  \tikz[baseline=-0.55ex]\draw[inclusion] (0,0)--(8mm,0);
  inverse-image construction
  \quad
  \tikz[baseline=-0.55ex]\draw[induced] (0,0)--(8mm,0);
  optimization
  \quad
  \tikz[baseline=-0.55ex]\draw[identify] (0,0)--(8mm,0);
  attained-value relation
};

\end{tikzpicture}%
}
\caption{From physical actuation to fiberwise optimization. The
propeller-state-to-wrench map factors through the actuator and allocation
maps. Minimizing the actuator-output cost over each allocation fiber defines
the infimal allocation cost and the minimum-output map.}
\label{fig:fiberwise-allocation-construction}
\end{figure*}

\fi
For
$\forceCoordinates,\momentCoordinates
\in\RealNumbers^{\spatialDimension}$, write
\begin{equation}
 \wrenchCoordinates
 =
 \begin{bmatrix}
  \forceCoordinates^{\transpose} &
  \momentCoordinates^{\transpose}
 \end{bmatrix}^{\transpose},
 \quad
 \allocationMatrix
 =
 \begin{bmatrix}
  \forceMatrix^{\transpose} &
  \momentMatrix^{\transpose}
 \end{bmatrix}^{\transpose}.
 \label{eq:wrench-partition}
\end{equation}

\begin{definition}[Infimal zero-moment allocation cost]
\label{def:zero-moment-allocation-value}
The infimal zero-moment allocation cost is the function
\begin{equation}
 \forceValue
 :=
 \wrenchValue\circ\zeroMomentEmbedding:
 \RealNumbers^{\spatialDimension}
 \to
 \NonnegativeReals,
 \quad
 \zeroMomentEmbedding(\forceCoordinates)
 :=
 \begin{bmatrix}
  \forceCoordinates^{\transpose} &
  \zeroMomentCoordinates^{\transpose}
 \end{bmatrix}^{\transpose}.
 \label{eq:force-value}
\end{equation}
\end{definition}

Thus, $\forceValue$ is the infimal allocation cost $\wrenchValue$ restricted to $\RealNumbers^{\spatialDimension}\times\{\zeroMomentCoordinates\}$, identified with $\RealNumbers^{\spatialDimension}$.

\subsection{Positive homogeneity}

\begin{definition}[Positive homogeneity]
\label{def:positive-homogeneity}
Let \(\GenericCone\) be a cone. A function
\(\genericFunction:\GenericCone\to\NonnegativeReals\) is positively
homogeneous of degree \(\homogeneityDegree>0\) if
\begin{equation}
 \genericFunction(\positiveScale\genericArgument)
 =\positiveScale^{\homogeneityDegree}
 \genericFunction(\genericArgument)
 \label{eq:positive-homogeneity}
\end{equation}
for every \(\genericArgument\in\GenericCone\) and
\(\positiveScale\geq0\).
\end{definition}

\begin{proposition}[Inheritance of homogeneity]
\label{prop:inheritance-homogeneity}
For the linear allocation map introduced in \cref{eq:allocation-map}, suppose
that \(\ActuatorOutputSpace\) is a cone and that \(\actuatorCost\) is
positively homogeneous of degree \(\homogeneityDegree>0\). Then
\begin{equation}
 \wrenchValue(\positiveScale\wrenchCoordinates)
 =\positiveScale^{\homogeneityDegree}
 \wrenchValue(\wrenchCoordinates),
 \quad \positiveScale\geq0, \;\; \wrenchCoordinates\in\WrenchSpace.
 \label{eq:wrench-value-homogeneity}
\end{equation}
The function \(\forceValue\) is positively homogeneous of the same degree.
\end{proposition}

\begin{proof}
By linearity of $\allocationMap$, for $\positiveScale>0$,
$\actuatorOutputArray\in\fiberAt{\wrenchCoordinates}$ if and only if
$\positiveScale\actuatorOutputArray
\in\fiberAt{\positiveScale\wrenchCoordinates}$. Hence,
$\fiberAt{\positiveScale\wrenchCoordinates}
=\positiveScale\fiberAt{\wrenchCoordinates}$.
Taking the infimum of the homogeneous cost over these corresponding
fibers proves \cref{eq:wrench-value-homogeneity}. At
$\positiveScale=0$, the zero actuator output is feasible and has zero
cost. Homogeneity of $\forceValue$ follows from its definition.
\end{proof}

\begin{remark}[Homogeneous costs are relevant]
Representative positively homogeneous costs include the weighted
\(\normOrder\)-norm
\(\actuatorCost(\actuatorOutputArray)
=\vectorNormOrder{\actuatorWeightMatrix\actuatorOutputArray}{\normOrder}\),
of degree one, and its positive power
\(\actuatorCost(\actuatorOutputArray)
=\vectorNormOrder{\actuatorWeightMatrix\actuatorOutputArray}{\normOrder}
^{\homogeneityDegree}\), of degree \(\homogeneityDegree\). A separable
power-law cost is
\(\actuatorCost(\actuatorOutputArray)
=\sum_{\actuatorIndex=1}^{\actuatorCount}
\actuatorCostWeight{\actuatorIndex}
\absoluteValue{\actuatorOutputComponent{\actuatorIndex}}^{\homogeneityDegree}\),
where \(\actuatorCostWeight{\actuatorIndex}>0\). Special cases are the
summed nonnegative actuator output
\(\actuatorCost(\actuatorOutputArray)
=\onesArray^{\transpose}\actuatorOutputArray\), (\(\homogeneityDegree=1\)); the squared
Euclidean cost
\(\actuatorCost(\actuatorOutputArray)
=\vectorNormOrder{\actuatorOutputArray}{2}^{2}\),  (\(\homogeneityDegree=2\)); and the propeller-drag power surrogate
\(\actuatorCost(\actuatorOutputArray)
=\sum_{\actuatorIndex=1}^{\actuatorCount}
\actuatorPowerCoefficient{\actuatorIndex}
\absoluteValue{\actuatorOutputComponent{\actuatorIndex}}^{3/2}\)
 of \cite{Aboudorra2024JINT}, with \(\homogeneityDegree=3/2\).
\end{remark}

\section{Minimum-Allocation Force Directions}
\label{sec:minimum-directions}

By \cref{ass:standing-actuation}, since $\actuatorCost$ is finite-valued
and nonnegative, $\forceValue$ is finite on $\UnitSphere$.

\begin{assumption}
\label{ass:directional-regularity}
The restriction of the infimal zero-moment allocation cost
$\forceValue$ to $\UnitSphere$ is lower semicontinuous.
\end{assumption}

\begin{definition}[Directional infimal allocation cost]
\label{def:directional-allocation-cost}
The directional infimal allocation cost is the function
\begin{equation}
 \directionalValue
 :=\left.\forceValue\right|_{\UnitSphere}:
 \UnitSphere\to\NonnegativeReals,
 \quad \text{thus} \quad
 \directionalValue(\unitForceDirection)
 :=\forceValue(\unitForceDirection).
 \label{eq:directional-value}
\end{equation}
Thus, $\directionalValue$ is the infimal allocation cost $\wrenchValue$ restricted to wrenches in $\UnitSphere\times\{\zeroMomentCoordinates\}$, identified with $\UnitSphere$.
Its minimum value on $\UnitSphere$ is
\begin{equation}
 \minimumDirectionalValue
 :=\minimum_{\unitForceDirection\in\UnitSphere}
 \directionalValue(\unitForceDirection)
 \in\NonnegativeReals,
 \label{eq:minimum-directional-value}
\end{equation}
and the corresponding minimum-level set is
\begin{equation}
 \minimumDirectionSet
 :=\directionalValue^{-1}\!\left(
 \left\{\minimumDirectionalValue\right\}\right)
 =\argminimum_{\unitForceDirection\in\UnitSphere}
 \directionalValue(\unitForceDirection)
 \subseteq\UnitSphere.
 \label{eq:minimum-direction-set}
\end{equation}
\end{definition}
\ifSHORTVERSION
\else
Figure~\ref{fig:directional-cost-construction} isolates the two restrictions
used above. The zero-moment embedding transfers the infimal wrench cost to
body-force space, while restriction to $\UnitSphere$ removes force magnitude
and leaves the directional landscape whose minimum level set is
$\minimumDirectionSet$.
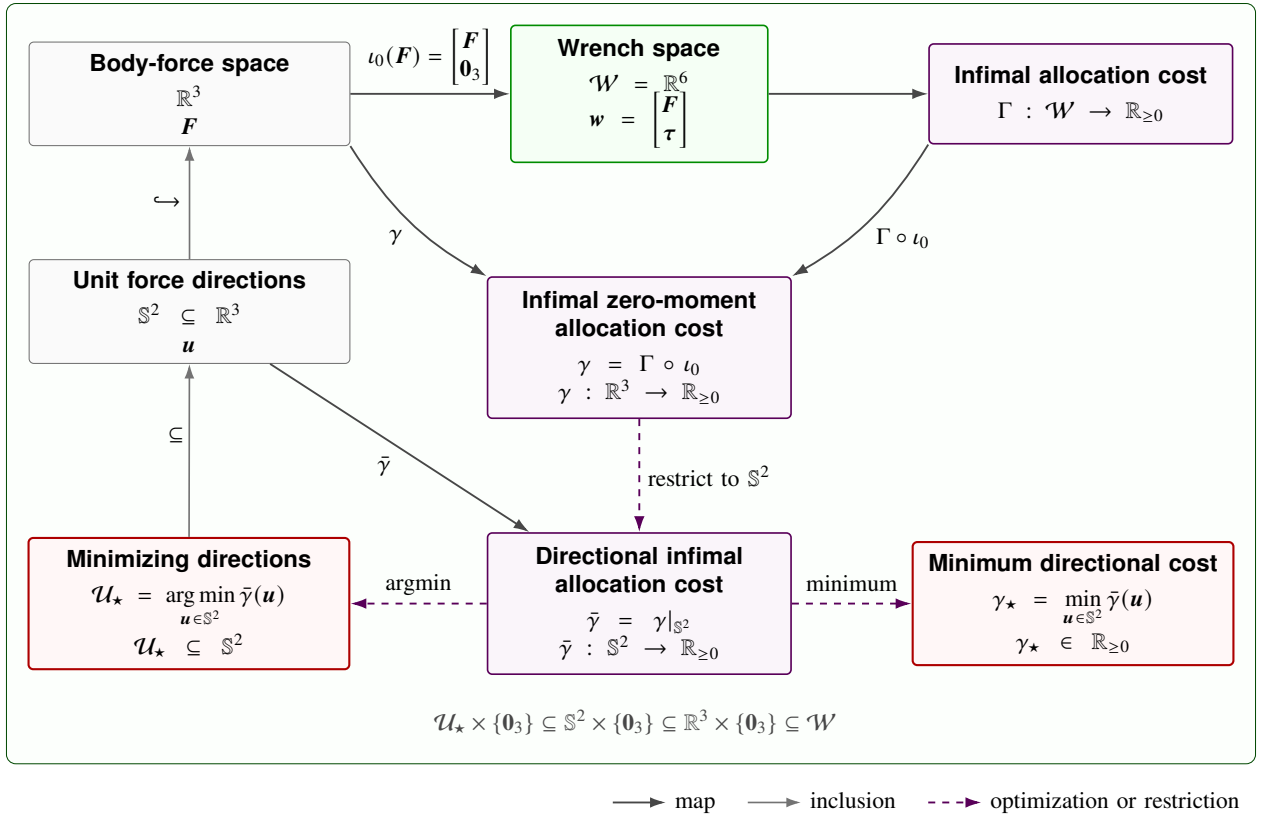
\begin{figure*}[t!]
\centering
\resizebox{0.95\textwidth}{!}{%
\begin{tikzpicture}[
  font=\sffamily,
  >=Latex,
  space/.style={
    rounded corners=2pt,
    draw=black!68,
    line width=0.7pt,
    fill=white,
    align=center,
    inner xsep=8pt,
    inner ysep=6pt,
    minimum height=15mm
  },
  wrench/.style={
    space,
    draw=green!55!black,
    fill=green!4,
    text width=3.3cm
  },
  cost/.style={
    space,
    draw=violet!70!black,
    fill=violet!4,
    text width=4.0cm
  },
  subset/.style={
    rounded corners=2pt,
    draw=black!58,
    fill=black!2,
    align=center,
    inner xsep=6pt,
    inner ysep=5pt,
    text width=4.4cm
  },
  selected/.style={
    subset,
    draw=red!68!black,
    fill=red!3,
    line width=0.9pt
  },
  map/.style={
    -{Latex[length=2.4mm,width=1.6mm]},
    line width=0.8pt,
    draw=black!72
  },
  inclusion/.style={
    -{Latex[length=2.2mm,width=1.5mm]},
    line width=0.7pt,
    draw=black!55
  },
  induced/.style={
    -{Latex[length=2.4mm,width=1.6mm]},
    line width=0.8pt,
    dashed,
    draw=violet!70!black
  },
  group/.style={
    rounded corners=4pt,
    draw=green!28!black,
    fill=green!1,
    inner sep=9pt
  },
  smallnote/.style={
    font=\normalsize,
    align=center,
    text=black!72
  }
]

\node[wrench] (W2) {
  {\bfseries\normalsize Wrench space}\par\smallskip
  $\WrenchSpace=\RealNumbers^{\fullWrenchDimension}$\par
  $\wrenchCoordinates=
   \begin{bmatrix}
    \forceCoordinates\\
    \momentCoordinates
   \end{bmatrix}$
};

\node[subset, left=24mm of W2] (R3) {
  {\bfseries\normalsize Body-force space}\par\smallskip
  $\RealNumbers^{\spatialDimension}$\par
  $\forceCoordinates$
};

\node[subset, below=17mm of R3] (S2) {
  {\bfseries\normalsize Unit force directions}\par\smallskip
  $\UnitSphere\subseteq\RealNumbers^{\spatialDimension}$\par
  $\unitForceDirection$
};

\node[cost, right=24mm of W2] (Gamma2) {
  {\bfseries\normalsize Infimal allocation cost}\par\smallskip
  $\wrenchValue:\WrenchSpace\to\NonnegativeReals$
};

\node[cost, below=17mm of W2] (gamma) {
  {\bfseries\normalsize Infimal zero-moment allocation cost}\par\smallskip
  $\forceValue=\wrenchValue\circ\zeroMomentEmbedding$\par
  $\forceValue:
   \RealNumbers^{\spatialDimension}\to\NonnegativeReals$
};

\node[cost, below=17mm of gamma] (gammabar) {
  {\bfseries\normalsize Directional infimal allocation cost}\par\smallskip
  $\directionalValue=\left.\forceValue\right|_{\UnitSphere}$\par
  $\directionalValue:\UnitSphere\to\NonnegativeReals$
};

\node[selected, left=20.5mm of gammabar] (Astar) {
  {\bfseries\normalsize Minimizing directions}\par\smallskip
  $\displaystyle
   \minimumDirectionSet
   =\argminimum_{\unitForceDirection\in\UnitSphere}
    \directionalValue(\unitForceDirection)$\par\smallskip
  $\minimumDirectionSet\subseteq\UnitSphere$
};

\node[selected, right=18mm of gammabar] (gammastar) {
  {\bfseries\normalsize Minimum directional cost}\par\smallskip
  $\displaystyle
   \minimumDirectionalValue
   =\minimum_{\unitForceDirection\in\UnitSphere}
    \directionalValue(\unitForceDirection)$\par\smallskip
  $\minimumDirectionalValue\in\NonnegativeReals$
};

\draw[map] (R3) -- node[above,font=\normalsize]
  {$\zeroMomentEmbedding(\forceCoordinates)
    =\begin{bmatrix}
      \forceCoordinates\\
      \zeroMomentCoordinates
     \end{bmatrix}$}
  (W2);
\draw[map] (W2) -- (Gamma2);
\draw[map] (R3.south east) to[bend right=14]
  node[below,font=\normalsize,xshift=-2mm] {$\forceValue$} (gamma.north west);
\draw[map] (Gamma2.south west) to[bend left=14]
  node[below,font=\normalsize,xshift=5mm]
  {$\wrenchValue\circ\zeroMomentEmbedding$} (gamma.north east);

\draw[inclusion] (S2) -- node[left,font=\normalsize] {$\hookrightarrow$} (R3);
\draw[map] (S2) -- node[below left,font=\normalsize]
  {$\directionalValue$} (gammabar);
\draw[induced] (gamma) -- node[right,font=\normalsize]
  {restrict to $\UnitSphere$} (gammabar);
\draw[inclusion] (Astar) -- node[above,font=\normalsize,xshift=-2mm]
  {$\subseteq$} (S2);
\draw[induced] (gammabar) -- node[above,font=\normalsize]
  {argmin} (Astar);
\draw[induced] (gammabar) -- node[above,font=\normalsize,yshift=0.5mm]
  {minimum} (gammastar);

\node[smallnote, below=4mm of gammabar] (nested) {
  $\minimumDirectionSet\times\{\zeroMomentCoordinates\}
   \subseteq
   \UnitSphere\times\{\zeroMomentCoordinates\}
   \subseteq
   \RealNumbers^{\spatialDimension}\times\{\zeroMomentCoordinates\}
   \subseteq\WrenchSpace$
};

\begin{scope}[on background layer]
  \node[group,
    fit=(R3)(S2)(W2)(Gamma2)(gamma)(gammabar)(Astar)(gammastar)(nested)]
    (restrictiongroup) {};
\end{scope}

\node[anchor=north east,font=\normalsize,align=left]
  at ($(restrictiongroup.south east)+(0,-3mm)$) {
  \tikz[baseline=-0.55ex]\draw[map] (0,0)--(8mm,0);
  map
  \quad
  \tikz[baseline=-0.55ex]\draw[inclusion] (0,0)--(8mm,0);
  inclusion
  \quad
  \tikz[baseline=-0.55ex]\draw[induced] (0,0)--(8mm,0);
  optimization or restriction
};

\end{tikzpicture}%
}
\caption{Restriction of the infimal allocation cost to zero-moment
wrenches and unit force directions. The resulting directional infimal
allocation cost determines the minimum directional cost and the set of
minimizing force directions.}
\label{fig:directional-cost-construction}
\end{figure*}

\fi

\begin{remark}
Compactness of \(\UnitSphere\) and lower semicontinuity of
\(\directionalValue\) guarantee that \(\minimumDirectionSet\) is nonempty. The
set is determined  by the actuation geometry, the zero-moment constraint,
and the actuator cost. Its elements need not coincide with a body
axis or a symmetry direction, and the set may contain one direction, several
directions, or a continuum of directions.
\end{remark}

\begin{proposition}[Magnitude--direction separation of the allocation cost]
\label{prop:magnitude-direction-separation}
For the linear allocation map introduced in \cref{eq:allocation-map}, suppose
that \(\ActuatorOutputSpace\) is a cone and that
\(\actuatorCost:\ActuatorOutputSpace\to\NonnegativeReals\) is positively
homogeneous of degree \(\homogeneityDegree>0\). Then, for every requested
body-force magnitude \(\requiredForceMagnitude\geq0\) and every unit body-force
direction \(\unitForceDirection\in\UnitSphere\),
\begin{equation}
 \forceValue(\requiredForceMagnitude\unitForceDirection)
 =\requiredForceMagnitude^{\homogeneityDegree}
 \directionalValue(\unitForceDirection).
 \label{eq:magnitude-direction-separation}
\end{equation}
Consequently, for every \(\requiredForceMagnitude>0\),
\begin{equation}
 \argminimum_{\unitForceDirection\in\UnitSphere}
 \forceValue(\requiredForceMagnitude\unitForceDirection)
 =\minimumDirectionSet.
 \label{eq:fixed-magnitude-minimizers}
\end{equation}
Thus, the directions minimizing the allocation cost are independent of the
requested body-force magnitude.
\end{proposition}

\begin{proof}
By \cref{prop:inheritance-homogeneity}, \(\forceValue\) is \(\homogeneityDegree\)-homogeneous. Hence,
\begin{equation}
 \forceValue(\requiredForceMagnitude\unitForceDirection)
 =\requiredForceMagnitude^{\homogeneityDegree}
 \forceValue(\unitForceDirection)
 =\requiredForceMagnitude^{\homogeneityDegree}
 \directionalValue(\unitForceDirection),
\end{equation}
which proves \cref{eq:magnitude-direction-separation}. For fixed positive
\(\requiredForceMagnitude\), the multiplicative factor is positive and
independent of \(\unitForceDirection\), so it does not change the minimizers.
\end{proof}
\ifSHORTVERSION
\else
The successive reductions are collected in
\cref{fig:complete-geometric-construction}. Fiberwise infimization first
transports the actuator-output cost to wrench space; zero-moment and
unit-sphere restrictions then expose the directional landscape; finally,
positive homogeneity reconstructs the solution on every spherical section
from the unit section. The figure also shows how a minimizing wrench is lifted
back to its minimum-output set.
\begin{figure}[t!]
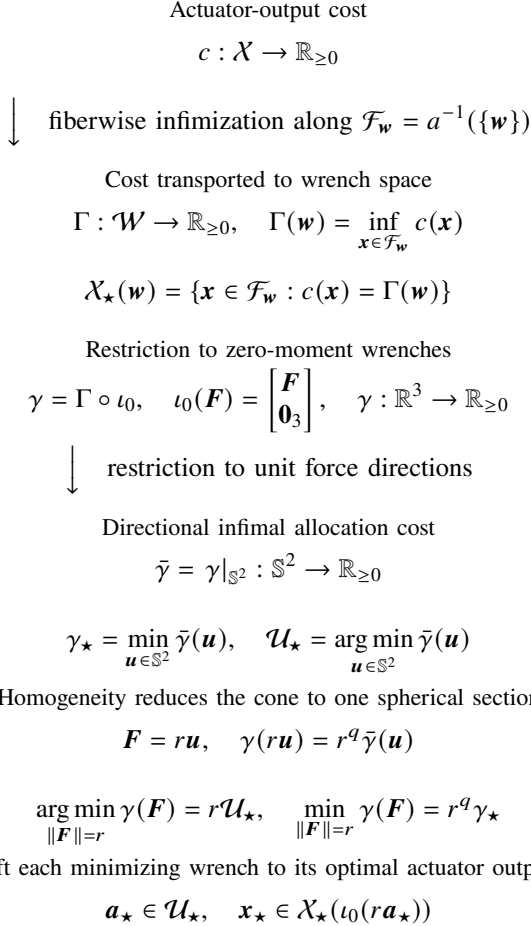

\centering
\begin{minipage}{0.99\columnwidth}
\centering

{\small Actuator-output cost}
\[
 \actuatorCost:\ActuatorOutputSpace\to\NonnegativeReals
\]
\vspace{-0.5em}
\[
 \Big\downarrow\quad\text{fiberwise infimization along }
 \fiberAt{\wrenchCoordinates}
 =\allocationMap^{-1}(\{\wrenchCoordinates\})
\]
\vspace{-0.5em}

{\small Cost transported to wrench space}
\[
 \wrenchValue:\WrenchSpace\to\NonnegativeReals,
 \quad
 \wrenchValue(\wrenchCoordinates)
 =\infimum_{\actuatorOutputArray\in\fiberAt{\wrenchCoordinates}}
 \actuatorCost(\actuatorOutputArray)
\]
\vspace{-0.5em}
\[
 \minimumOutputSet(\wrenchCoordinates)
 =\left\{
 \actuatorOutputArray\in\fiberAt{\wrenchCoordinates}:
 \actuatorCost(\actuatorOutputArray)=\wrenchValue(\wrenchCoordinates)
 \right\}
\]
\vspace{-0.5em}

{\small Restriction to zero-moment wrenches}
\[
 \forceValue
 =\wrenchValue\circ\zeroMomentEmbedding,
 \quad
 \zeroMomentEmbedding(\forceCoordinates)
 =\begin{bmatrix}
   \forceCoordinates\\
   \zeroMomentCoordinates
  \end{bmatrix},
 \quad
 \forceValue:\RealNumbers^{\spatialDimension}\to\NonnegativeReals
\]
\vspace{-0.5em}
\[
 \Big\downarrow\quad\text{restriction to unit force directions}
\]
\vspace{-0.5em}

{\small Directional infimal allocation cost}
\[
 \directionalValue
 =\left.\forceValue\right|_{\UnitSphere}:
 \UnitSphere\to\NonnegativeReals
\]
\vspace{-0.5em}
\[
 \minimumDirectionalValue
 =\minimum_{\unitForceDirection\in\UnitSphere}
 \directionalValue(\unitForceDirection),
 \quad
 \minimumDirectionSet
 =\argminimum_{\unitForceDirection\in\UnitSphere}
 \directionalValue(\unitForceDirection)
\]
\vspace{-0.5em}

{\small Homogeneity reduces the cone to one spherical section}
\[
 \forceCoordinates
 =\requiredForceMagnitude\unitForceDirection,
 \quad
 \forceValue(\requiredForceMagnitude\unitForceDirection)
 =\requiredForceMagnitude^{\homogeneityDegree}
 \directionalValue(\unitForceDirection)
\]
\vspace{-0.5em}
\[
 \argminimum_{\vectorNorm{\forceCoordinates}=\requiredForceMagnitude}
 \forceValue(\forceCoordinates)
 =\requiredForceMagnitude\minimumDirectionSet,
 \quad
 \minimum_{\vectorNorm{\forceCoordinates}=\requiredForceMagnitude}
 \forceValue(\forceCoordinates)
 =\requiredForceMagnitude^{\homogeneityDegree}
 \minimumDirectionalValue
\]
\vspace{-0.5em}

{\small Lift each minimizing wrench to its optimal actuator outputs}
\[
 \minimumForceDirection\in\minimumDirectionSet,
 \quad
 \actuatorOutputArray_{\star}
 \in
 \minimumOutputSet\!\left(
 \zeroMomentEmbedding(
 \requiredForceMagnitude\minimumForceDirection)
 \right)
\]

\end{minipage}
\caption{Complete geometric construction. Fiberwise infimization transports the
actuator-output cost from the actuator-output space to the wrench space.
Restriction to zero-moment wrenches and unit force directions yields the
directional infimal allocation cost. Positive homogeneity makes every spherical
section of the zero-moment force cone a radially scaled copy of the unit section,
and optimal actuator outputs are recovered by lifting the minimizing wrenches
through the minimum-output map.}
\label{fig:complete-geometric-construction}
\end{figure}

\fi

\section{Emergent Underactuation}
\label{sec:emergent-underactuation}

To embed the allocation model in the standard multirotor rigid-body
dynamics \cite{Michieletto2018TRO}, let $\WorldFrame$ and $\BodyFrame$
denote the world and body frames. The attitude of $\BodyFrame$ relative
to $\WorldFrame$ is represented by $\attitude\in\RotationGroup$.
Accordingly, if a free vector has coordinate array $\forceCoordinates$
in $\BodyFrame$, its coordinate array in $\WorldFrame$ is
$\attitude\forceCoordinates$.
The third axis of \(\WorldFrame\) is vertical and oriented opposite to the
gravitational acceleration. Its coordinate array in \(\WorldFrame\) is
\(\worldVertical=\worldVerticalCoordinates\), so the gravitational acceleration
has coordinate array \(-\gravityAcceleration\worldVertical\), with
\(\gravityAcceleration>0\). Let
\(\vehicleMass\) denote the  vehicle mass.
Let
\(\positionTrajectory:[0,\timeHorizon]
\to\RealNumbers^{\spatialDimension}\)
be a prescribed twice continuously differentiable position trajectory, where
\(\positionTrajectory(\timeCoordinate)\) is the position coordinate array of
the body-frame origin expressed in \(\WorldFrame\). The required
 total
propeller force trajectory expressed in \(\WorldFrame\) is
 \(\requiredWorldForce:[0,\timeHorizon]
\to\RealNumbers^{\spatialDimension}\)  defined by
\begin{equation}
 \requiredWorldForce(\timeCoordinate)
 :=\vehicleMass\ddot{\positionTrajectory}(\timeCoordinate)
 +\vehicleMass\gravityAcceleration\worldVertical.
 \label{eq:required-world-force}
\end{equation}
Assume that \(\requiredWorldForce(\timeCoordinate)\neq\zeroForceCoordinates\)
for every \(\timeCoordinate\in[0,\timeHorizon]\). Define
\(\requiredForceMagnitude:[0,\timeHorizon]\to\PositiveReals\) and
\(\requiredForceDirection:[0,\timeHorizon]\to\UnitSphere\) by
\begin{equation}
 \requiredForceMagnitude(\timeCoordinate)
 :=\vectorNorm{\requiredWorldForce(\timeCoordinate)},
 \quad
 \requiredForceDirection(\timeCoordinate)
 :=\requiredWorldForce(\timeCoordinate)\;/\;
\requiredForceMagnitude(\timeCoordinate).
 \label{eq:required-force-decomposition}
\end{equation}
For an attitude trajectory
\(\attitude:[0,\timeHorizon]
\to\RotationGroup\), the same force in body
coordinates is \(\requiredBodyForce:[0,\timeHorizon]
\to\RealNumbers^{\spatialDimension}\) defined by
\begin{equation}
 \requiredBodyForce(\timeCoordinate)
 =\attitude(\timeCoordinate)^{\transpose}
 \requiredWorldForce(\timeCoordinate)
 =\requiredForceMagnitude(\timeCoordinate)
 \attitude(\timeCoordinate)^{\transpose}
 \requiredForceDirection(\timeCoordinate).
 \label{eq:required-body-force}
\end{equation}
Thus, the position trajectory fixes the force magnitude $\requiredForceMagnitude$ and world-frame
direction $\requiredForceDirection$, while attitude trajectory selects its body-frame direction \(\attitude(\timeCoordinate)^{\transpose}
 \requiredForceDirection(\timeCoordinate)\).
For each time, define
\begin{equation}
 \optimalAttitudeSet(\timeCoordinate)
 :=\left\{
 \attitude\in\RotationGroup:
 \attitude^{\transpose}\requiredForceDirection(\timeCoordinate)
 \in\minimumDirectionSet
 \right\}.
 \label{eq:minimum-attitude-set}
\end{equation}
Thus,
\(\optimalAttitudeSet:[0,\timeHorizon]
\rightrightarrows\RotationGroup\)
is a set-valued map assigning to each time the attitudes that represent the
required force along a direction in \(\minimumDirectionSet\).

\begin{theorem}[Minimum-cost attitude family]
\label{thm:minimum-cost-attitude-family}
Suppose that the assumptions of
\cref{prop:magnitude-direction-separation} hold and that
$\directionalValue$ is lower semicontinuous on $\UnitSphere$. Then every
\(\optimalAttitude(\timeCoordinate)
\in\optimalAttitudeSet(\timeCoordinate)\) minimizes the instantaneous
zero-moment allocation cost in \(\RotationGroup\):
\begin{equation}
\begin{aligned}
 \forceValue\!\left(
 \optimalAttitude(\timeCoordinate)^{\transpose}
 \requiredWorldForce(\timeCoordinate)\right)
 =
 \minimum_{\attitude\in\RotationGroup}
 \forceValue\!\left(
 \attitude^{\transpose}\requiredWorldForce(\timeCoordinate)\right)
 =\requiredForceMagnitude(\timeCoordinate)^{\homogeneityDegree}
 \minimumDirectionalValue.
\end{aligned}
 \label{eq:instantaneous-attitude-optimum}
\end{equation}
Any measurable selection
\(\optimalAttitude:[0,\timeHorizon]\to\RotationGroup\) satisfying
\(\optimalAttitude(\timeCoordinate)
\in\optimalAttitudeSet(\timeCoordinate)\) almost everywhere and having
integrable allocation cost also minimizes
\begin{equation}
 \trajectoryCostOf{\attitude}{\positionTrajectory}
 :=\int_{0}^{\timeHorizon}
 \forceValue\!\left(
 \attitude(\timeCoordinate)^{\transpose}
 \requiredWorldForce(\timeCoordinate)\right)
 \differentialTime
 \label{eq:trajectory-cost}
\end{equation}
over all measurable attitude trajectories having integrable allocation cost.
\end{theorem}

\begin{proof}
Any \(\attitude\in\RotationGroup\), preserves the force
magnitude. Hence
\(\attitude^{\transpose}\requiredWorldForce
=\requiredForceMagnitude\unitForceDirection\), where
\(\unitForceDirection
=\attitude^{\transpose}\requiredForceDirection\in\UnitSphere\). By
\cref{eq:magnitude-direction-separation},
\begin{equation}
 \forceValue\!\left(\attitude^{\transpose}\requiredWorldForce\right)
 =\requiredForceMagnitude^{\homogeneityDegree}
 \directionalValue(\unitForceDirection)
 \geq\requiredForceMagnitude^{\homogeneityDegree}
 \minimumDirectionalValue.
\end{equation}
Equality holds exactly when
\(\unitForceDirection\in\minimumDirectionSet\), equivalently when
\(\attitude\in\optimalAttitudeSet\). Integrating the pointwise inequality
proves the trajectory statement.
\end{proof}

\begin{corollary}[Emergent underactuation]
\label{cor:emergent-underactuation}
Fix any \(\minimumForceDirection\in\minimumDirectionSet\). Every attitude
trajectory \(\optimalAttitude:[0,\timeHorizon]\to\RotationGroup\)  satisfying
\begin{equation}
 \optimalAttitude(\timeCoordinate)\minimumForceDirection
 =\requiredForceDirection(\timeCoordinate)
 \label{eq:alignment-law}
\end{equation}
is instantaneously minimum cost. Its required body force is
\begin{equation}
 \requiredBodyForce(\timeCoordinate)
 =\requiredForceMagnitude(\timeCoordinate)\minimumForceDirection,
 \label{eq:underactuated-body-force}
\end{equation}
and therefore
\begin{equation}
 \left(\forceProjectorComplement\right)\requiredBodyForce(\timeCoordinate)
 =\zeroForceCoordinates.
 \label{eq:zero-transverse-force}
\end{equation}
\end{corollary}

\begin{proof}
Because \(\minimumForceDirection\in\minimumDirectionSet\),
\cref{eq:alignment-law} implies
\(\optimalAttitude\in\optimalAttitudeSet\) and
\(\optimalAttitude^{\transpose}\requiredForceDirection
=\minimumForceDirection\). Substitution into
\cref{eq:required-body-force} proves
\cref{eq:underactuated-body-force}. Since
\(\forceProjectorComplement\minimumForceDirection
=\zeroForceCoordinates\), \cref{eq:zero-transverse-force} follows.
\end{proof}

\begin{remark}[Underactuation induced by optimality]
Equation~\eqref{eq:zero-transverse-force} is the force constraint of a platform
that can generate force only along the fixed body direction
\(\minimumForceDirection\). This one-dimensional force structure is not
imposed by the actuation rank. It is selected by minimum cost from the complete
force authority of a fully actuated platform.
\end{remark}

\begin{remark}[Nonunique minimizing directions]
The set \(\minimumDirectionSet\) need not be a singleton. A fixed
underactuated-like structure is obtained by selecting any
\(\minimumForceDirection\in\minimumDirectionSet\) and using the same body-fixed
direction throughout the motion. A time-varying selection remains
instantaneously minimum cost but need not confine the body force to one fixed
body-frame subspace.
\end{remark}

\begin{remark}[Regularity and dynamic realizability]
For fixed $\minimumForceDirection$ and nonvanishing required force, a
$C^3$ position trajectory permits a $C^1$ minimum-cost attitude
selection and continuous angular velocity, while a $C^4$ trajectory
permits a $C^2$ selection and continuous angular acceleration. Dynamic
realizability additionally requires satisfaction of the angular-rate,
moment, and actuator constraints, which lie outside the instantaneous
zero-moment allocation result.
\end{remark}

\subsection{Illustrative example and parallel-thrust interpretation}

\begin{example}[Planar transfer]
Consider
$\positionTrajectory(\timeCoordinate)
=\horizontalPosition(\timeCoordinate)\worldAxisOne
+\constantAltitude\worldVertical$
and suppose that the chosen allocation geometry and cost yield
$\bodyAxisThree\in\minimumDirectionSet$. Selecting
$\minimumForceDirection=\bodyAxisThree$ gives
\begin{equation}
\begin{aligned}
 \requiredWorldForce(\timeCoordinate)
 &=\vehicleMass\horizontalAcceleration(\timeCoordinate)\worldAxisOne
 +\vehicleMass\gravityAcceleration\worldVertical,
 \\
 \requiredForceMagnitude(\timeCoordinate)
 &=\vehicleMass\sqrt{\gravityAcceleration^{2}
 +\horizontalAcceleration(\timeCoordinate)^{2}}.
\end{aligned}
 \label{eq:planar-required-force}
\end{equation}
The minimum-cost attitude satisfies
\begin{equation}
 \optimalAttitude(\timeCoordinate)\bodyAxisThree
 =\left(\horizontalAcceleration(\timeCoordinate)\worldAxisOne
 +\gravityAcceleration\worldVertical\right)\big/
 \left(\sqrt{\gravityAcceleration^{2}
 +\horizontalAcceleration(\timeCoordinate)^{2}}\right).
 \label{eq:planar-alignment}
\end{equation}
By contrast, let $\horizontalAttitude\in\RotationGroup$ satisfy
$\horizontalAttitude\bodyAxisOne=\worldAxisOne$ and
$\horizontalAttitude\bodyAxisThree=\worldVertical$. Then
\begin{equation}
 \horizontalAttitude^{\transpose}
 \requiredWorldForce(\timeCoordinate)
 =\vehicleMass\horizontalAcceleration(\timeCoordinate)\bodyAxisOne
 +\vehicleMass\gravityAcceleration\bodyAxisThree.
 \label{eq:horizontal-body-force}
\end{equation}
This realization has no smaller instantaneous allocation cost and is
strictly more costly whenever its unit body-force direction does not
belong to $\minimumDirectionSet$. The familiar tilting motion of an
underactuated multirotor is therefore recovered as a minimum-cost
realization of the same prescribed translation on a fully actuated
platform.
\end{example}

\begin{remark}[Parallel-thrust lower bound]
\label{rem:parallel-thrust-lower-bound}
For the summed nonnegative actuator-output cost
$\actuatorCost(\actuatorOutputArray)
=\onesArray^{\transpose}\actuatorOutputArray$, the triangle inequality
gives
$\forceValue(\forceCoordinates)\geq
\vectorNorm{\forceCoordinates}$.
Equality holds when a feasible zero-moment allocation uses only rotor
forces parallel and codirectional with $\forceCoordinates$. Hence, a
parallel-axis platform aligned with the required world-frame force
attains
$\int_{0}^{\timeHorizon}
\vectorNorm{\requiredWorldForce(\timeCoordinate)}
\differentialTime$,
the universal lower bound for the same ideal unidirectional force
generators. This comparison excludes differences in vehicle mass,
rotor aerodynamics, actuator limits, and attitude-transition cost.
\end{remark}

\section{Cost Ordering of Full-Pose Trajectories}
\label{sec:trajectory-ordering}

For a prescribed position trajectory $\positionTrajectory$, let
$\AttitudeTrajectorySpace$ be the set of measurable maps
$\attitude:[0,\timeHorizon]\to\RotationGroup$ for which
$\trajectoryCostOf{\attitude}{\positionTrajectory}$ is finite. Each
$\attitude\in\AttitudeTrajectorySpace$ parametrizes the full-pose
trajectory $(\positionTrajectory,\attitude):[0,\timeHorizon]\to\RealNumbers^{\spatialDimension}\times\RotationGroup$.

\begin{definition}[Allocation optimality gap]
\label{def:allocation-optimality-gap}
For every $\attitude\in\AttitudeTrajectorySpace$, the instantaneous
allocation optimality gap is
\begin{equation}
 \instantaneousGapOf{\attitude}(\timeCoordinate)
 :=
 \requiredForceMagnitude(\timeCoordinate)^{\homogeneityDegree}
 \left[
 \directionalValue\!\left(
 \attitude(\timeCoordinate)^{\transpose}
 \requiredForceDirection(\timeCoordinate)\right)
 -
 \minimumDirectionalValue
 \right].
 \label{eq:instantaneous-gap}
\end{equation}
Its accumulated value is
\begin{equation}
 \accumulatedGapOf{\attitude}{\positionTrajectory}
 :=
 \int_{0}^{\timeHorizon}
 \instantaneousGapOf{\attitude}(\timeCoordinate)
 \differentialTime.
 \label{eq:accumulated-gap}
\end{equation}
\end{definition}

In general, \(\instantaneousGapOf{\attitude}(\timeCoordinate)\geq 0\) and
\(\instantaneousGapOf{\attitude}(\timeCoordinate)=0
\Leftrightarrow
\attitude(\timeCoordinate)^{\transpose}
\requiredForceDirection(\timeCoordinate)\in\minimumDirectionSet\).

\begin{proposition}[Exact allocation-cost decomposition]
\label{prop:exact-cost-decomposition}
Every $\attitude\in\AttitudeTrajectorySpace$ satisfies
\begin{equation}
 \trajectoryCostOf{\attitude}{\positionTrajectory}
 =
 \minimumTrajectoryCostOf{\positionTrajectory}
 +
 \accumulatedGapOf{\attitude}{\positionTrajectory},
 \label{eq:cost-decomposition}
\end{equation}
where
\begin{equation}
 \minimumTrajectoryCostOf{\positionTrajectory}
 :=
 \minimumDirectionalValue
 \int_{0}^{\timeHorizon}
 \requiredForceMagnitude(\timeCoordinate)^{\homogeneityDegree}
 \differentialTime.
 \label{eq:minimum-trajectory-cost}
\end{equation}
\end{proposition}

\begin{proof}
Magnitude--direction separation gives the integrand as
$\requiredForceMagnitude(\timeCoordinate)^{\homogeneityDegree}$ times
the directional infimal allocation cost evaluated at
$\attitude(\timeCoordinate)^{\transpose}
\requiredForceDirection(\timeCoordinate)$. Adding and subtracting
$\requiredForceMagnitude(\timeCoordinate)^{\homogeneityDegree}
\minimumDirectionalValue$ and integrating proves
\cref{eq:cost-decomposition}.
\end{proof}

The decomposition separates the unavoidable minimum cost intrinsic to
the position motion from the exact attitude-dependent penalty. The
former is common to all full-pose trajectories sharing
$\positionTrajectory$, while the accumulated gap orders their attitude
choices under the instantaneous zero-moment allocation model.

\begin{definition}[Cost-induced preorder at fixed position motion]
\label{def:cost-induced-attitude-preorder}
For $\attitudeOne,\attitudeTwo\in\AttitudeTrajectorySpace$, define
\begin{equation}
\begin{aligned}
 \attitudeOne\attitudePreorder\attitudeTwo
 \quad\Longleftrightarrow\quad&
 \accumulatedGapOf{\attitudeOne}{\positionTrajectory}
 \leq
 \accumulatedGapOf{\attitudeTwo}{\positionTrajectory}.
\end{aligned}
 \label{eq:attitude-preorder}
\end{equation}
They are cost equivalent when
\begin{equation}
\begin{aligned}
 \attitudeOne\attitudeCostEquivalence\attitudeTwo
 \quad\Longleftrightarrow\quad&
 \accumulatedGapOf{\attitudeOne}{\positionTrajectory}
 =
 \accumulatedGapOf{\attitudeTwo}{\positionTrajectory}.
\end{aligned}
 \label{eq:attitude-cost-equivalence}
\end{equation}
\end{definition}

\begin{proposition}[Preorder and induced quotient order]
\label{prop:cost-ordering}
The relation $\attitudePreorder$ is a total preorder on
$\AttitudeTrajectorySpace$ and satisfies
\begin{equation}
\begin{aligned}
 \attitudeOne\attitudePreorder\attitudeTwo
 \quad\Longleftrightarrow\quad&
 \trajectoryCostOf{\attitudeOne}{\positionTrajectory}
 \leq
 \trajectoryCostOf{\attitudeTwo}{\positionTrajectory}.
\end{aligned}
 \label{eq:cost-order-equivalence}
\end{equation}
The relation $\attitudeCostEquivalence$ partitions the corresponding
full-pose trajectories into equal-cost classes. The induced relation
\begin{equation}
 [\attitudeOne]\attitudeClassOrder[\attitudeTwo]
 \quad\Longleftrightarrow\quad
 \attitudeOne\attitudePreorder\attitudeTwo
 \label{eq:quotient-cost-order}
\end{equation}
is a total order on $\AttitudeQuotientSpace$.
\end{proposition}

\begin{proof}
Ordering the accumulated gaps gives a total preorder. Since
\cref{eq:cost-decomposition} adds the common term
$\minimumTrajectoryCostOf{\positionTrajectory}$, it also proves
\cref{eq:cost-order-equivalence}. Identifying equal-gap trajectories
makes \cref{eq:quotient-cost-order} antisymmetric and therefore a total
order.
\end{proof}

\begin{proposition}[Minimum equivalence class]
\label{prop:minimum-equivalence-class}
Define
\begin{equation}
\begin{aligned}
 \OptimalAttitudeTrajectorySet
 &:=
 {\bigl\{}
 \attitude\in\AttitudeTrajectorySpace:
 \attitude(\timeCoordinate)^{\transpose}
 \requiredForceDirection(\timeCoordinate)
 \in\minimumDirectionSet,
 \\
 &\hspace{2.8em}
 \text{for almost every }
 \timeCoordinate\in[0,\timeHorizon]
 {\bigr\}}.
\end{aligned}
 \label{eq:optimal-attitude-trajectory-set}
\end{equation}
Then
\begin{equation}
\begin{aligned}
 \attitude\in\OptimalAttitudeTrajectorySet
 \quad\Longleftrightarrow\quad&
 \accumulatedGapOf{\attitude}{\positionTrajectory}
 =
 \zeroScalar
 \\
 \quad\Longleftrightarrow\quad&
 \trajectoryCostOf{\attitude}{\positionTrajectory}
 =
 \minimumTrajectoryCostOf{\positionTrajectory}.
\end{aligned}
 \label{eq:minimum-class-characterization}
\end{equation}
Consequently, $\OptimalAttitudeTrajectorySet$ parametrizes the minimum
full-pose equivalence class over $\positionTrajectory$.
\end{proposition}

\begin{proof}
The nonnegative instantaneous gap has zero integral exactly when it
vanishes almost everywhere. This is equivalent to membership in
$\OptimalAttitudeTrajectorySet$, while the cost equality follows from
\cref{eq:cost-decomposition}.
\end{proof}

\begin{remark}[The price of rotational freedom]
Each full-pose trajectory $(\positionTrajectory,\attitude)$ determines
the body-frame unit-force-direction trajectory
\begin{equation}
 \timeCoordinate
 \longmapsto
 \attitude(\timeCoordinate)^{\transpose}
 \requiredForceDirection(\timeCoordinate)
 \in\UnitSphere.
 \label{eq:body-force-direction-trajectory}
\end{equation}
The accumulated gap is the weighted integral of the excess directional
cost along this spherical trajectory. The minimum class consists
exactly of trajectories remaining in $\minimumDirectionSet$ almost
everywhere. In particular, fixing
$\minimumForceDirection\in\minimumDirectionSet$ and imposing
$\attitude(\timeCoordinate)\minimumForceDirection
=\requiredForceDirection(\timeCoordinate)$ yields a representative.
By \cref{eq:underactuated-body-force,eq:zero-transverse-force}, its body
force remains in $\linearSpan\{\minimumForceDirection\}$, making
emergent underactuation a minimum-cost realization of the prescribed
position motion. Full actuation permits other attitude trajectories
without changing that translational motion, and the accumulated gap gives their exact
additional allocation cost. It thereby quantifies the price of using
the independent rotational freedom provided by full actuation.
\end{remark}

\begin{remark}[Relative allocation optimality gap]
When $\minimumTrajectoryCostOf{\positionTrajectory}>0$, define
\begin{equation}
 \relativeGapOf{\attitude}{\positionTrajectory}
 :=
 \accumulatedGapOf{\attitude}{\positionTrajectory}
 \;/\;
 \minimumTrajectoryCostOf{\positionTrajectory}.
 \label{eq:relative-allocation-gap}
\end{equation}
This dimensionless quantity is the relative cost increase induced by
the attitude component above the minimum allocation cost compatible
with the position motion. Hence,
$100\relativeGapOf{\attitude}{\positionTrajectory}\%$ is the
corresponding percentage increase. It  enables  comparison of full-pose trajectories with different position components. The resulting global ordering and its invariance properties are left for future study.
\end{remark}

\section{Full Actuation as Reserve Wrench Authority}
\label{sec:resilient-full-actuation}

The minimum-cost nominal realization in
\cref{eq:underactuated-body-force} uses a single body-fixed force
direction even though the platform has complete wrench capability.
Thus, transverse force directions remain unused in the nominal
zero-moment allocation. This suggests a complementary interpretation of
full actuation as reserve wrench authority around an
underactuated-like minimum-cost realization.

Consider an external force
$\externalWorldForce(\timeCoordinate)$, expressed in $\WorldFrame$ and
applied without an associated body moment. The disturbed translational
dynamics are
\begin{equation}
 \vehicleMass\ddot{\positionTrajectory}(\timeCoordinate)
 =
 \generatedWorldForce(\timeCoordinate)
 +
 \externalWorldForce(\timeCoordinate)
 -
 \vehicleMass\gravityAcceleration\worldVertical.
 \label{eq:disturbed-translation-dynamics}
\end{equation}
Exact preservation of the prescribed position trajectory requires
\begin{equation}
 \generatedWorldForce(\timeCoordinate)
 =
 \disturbedRequiredWorldForce(\timeCoordinate),
 \label{eq:exact-disturbance-compensation}
\end{equation}
where
\begin{equation}
 \disturbedRequiredWorldForce(\timeCoordinate)
 :=
 \requiredWorldForce(\timeCoordinate)
 -
 \externalWorldForce(\timeCoordinate).
 \label{eq:disturbed-required-force}
\end{equation}
Whenever
$\disturbedRequiredWorldForce(\timeCoordinate)
\neq\zeroForceCoordinates$, define
\begin{equation}
 \disturbedForceMagnitude(\timeCoordinate)
 :=
 \vectorNorm{\disturbedRequiredWorldForce(\timeCoordinate)},
 \quad
 \disturbedForceDirection(\timeCoordinate)
 :=
 {\disturbedRequiredWorldForce(\timeCoordinate)}
\big /{\disturbedForceMagnitude(\timeCoordinate)}.
 \label{eq:disturbed-force-decomposition}
\end{equation}

\begin{remark}
At its current attitude, a fully actuated platform can immediately
generate the disturbed required force. A structurally underactuated
platform can do so only if that force is aligned with its body-fixed
thrust direction. If angular velocity is bounded, any required
reorientation takes nonzero time, during which tracking error may
accumulate. Under exact disturbance knowledge and negligible actuation
and control delays, full actuation removes this delay through its
transverse force authority.
\end{remark}

Minimum cost is defined relative to the adopted allocation cost and the
required force. Thus, the nominal attitude is minimum-cost for the
undisturbed required force, but need not remain so after a disturbance.
The disturbed minimum-allocation attitude family is
\begin{equation}
 \disturbedOptimalAttitudeSet(\timeCoordinate)
 :=
 \left\{
 \attitude\in\RotationGroup:
 \attitude^{\transpose}
 \disturbedForceDirection(\timeCoordinate)
 \in\minimumDirectionSet
 \right\}.
 \label{eq:disturbed-optimal-attitude-set}
\end{equation}

\begin{proposition}[Reserve-authority cost decomposition]
\label{prop:reserve-authority-cost-decomposition}
Let $\optimalAttitude(\timeCoordinate)$ be minimum-cost for the nominal
required force. Full actuation permits generation of
$\disturbedRequiredWorldForce(\timeCoordinate)$ at this attitude without
prior reorientation. Define
\begin{equation}
\begin{aligned}
 \fixedAttitudeGap(\timeCoordinate)
 &:=
 \disturbedForceMagnitude(\timeCoordinate)^{\homogeneityDegree}
 \bigl[
 \directionalValue\!\left(
 \optimalAttitude(\timeCoordinate)^{\transpose}
 \disturbedForceDirection(\timeCoordinate)
 \right)
 -
 \minimumDirectionalValue
 \bigr].
\end{aligned}
 \label{eq:fixed-attitude-disturbance-gap}
\end{equation}
Then
\begin{equation}
 \forceValue\!\left(
 \optimalAttitude(\timeCoordinate)^{\transpose}
 \disturbedRequiredWorldForce(\timeCoordinate)
 \right)
 =
 \disturbedForceMagnitude(\timeCoordinate)^{\homogeneityDegree}
 \minimumDirectionalValue
 +
 \fixedAttitudeGap(\timeCoordinate).
 \label{eq:disturbed-cost-decomposition}
\end{equation}
Moreover, $\fixedAttitudeGap(\timeCoordinate)\geq0$, with equality if
and only if
$\optimalAttitude(\timeCoordinate)
\in\disturbedOptimalAttitudeSet(\timeCoordinate)$.
\end{proposition}

\begin{proof}
Magnitude--direction separation gives the left-hand side of
\cref{eq:disturbed-cost-decomposition} as
$\disturbedForceMagnitude^{\homogeneityDegree}$
times
$\directionalValue(
\optimalAttitude^{\transpose}\disturbedForceDirection)$.
Adding and subtracting
$\disturbedForceMagnitude^{\homogeneityDegree}
\minimumDirectionalValue$
gives the decomposition, whose second term has the form of the
allocation optimality gap in \cref{eq:instantaneous-gap}.
\Cref{eq:minimum-directional-value,eq:minimum-direction-set} prove its
nonnegativity and zero condition.
\end{proof}

\section{Numerical Illustration}
\label{sec:numerical-illustration}

We use a planar restriction of the three-dimensional model to compare three realizations of the same horizontal position trajectory:
 a parallel-thrust underactuated platform (UA), a
fully actuated platform rotating toward the instantaneous
minimum-allocation attitude (FA-opt), and a fully actuated platform
maintaining the constant nonoptimal pitch $20\,\mathrm{deg}$ (FA-const).
The fully actuated platforms generate the complete commanded force
instantaneously, whereas UA generates its projection onto the current
body-fixed thrust direction. The simulation uses a $12\,\mathrm{s}$
horizon, mass
$1.8\,\mathrm{kg}$, gravitational acceleration
$9.81\,\mathrm{m}\,\mathrm{s}^{-2}$, and a $35\,\mathrm{N}$ force
limit. The horizontal velocity increases smoothly from zero to
$0.6\,\mathrm{m}\,\mathrm{s}^{-1}$ over the first $2\,\mathrm{s}$,
remains constant until $10\,\mathrm{s}$, and then decreases
symmetrically to zero. A $-4\,\mathrm{N}$ horizontal external force
acts from $4\,\mathrm{s}$ to $7\,\mathrm{s}$, with
$0.2\,\mathrm{s}$ half-cosine transitions.

The commanded force combines exact disturbance feedforward with
proportional--derivative position feedback having gain matrices
$\operatorname{diag}(3,10)$ and $\operatorname{diag}(4,6)$. For UA and
FA-opt, the desired pitch aligns the preferred body-force direction
with the commanded force and is limited to $45\,\mathrm{deg}$. The common
pitch dynamics use gains $35$ and $12$, with angular-velocity and
angular-acceleration limits of
$55\,\mathrm{deg}\,\mathrm{s}^{-1}$ and
$180\,\mathrm{deg}\,\mathrm{s}^{-2}$. The UA cost is the squared thrust magnitude. For the fully actuated
cases, the allocation map and actuator cost yield the zero-moment force
value
$\forceValue(\forceCoordinates)
=12(F_x^B)^2+1.15(F_z^B)^2$,
so force along the body $z$-axis is less costly than transverse force.
This illustrative force value excludes the body moment required for
attitude motion, consistently with the zero-moment scope of the theory.

Figure~\ref{fig:numerical-comparison} summarizes the results. The first
row shows that both fully actuated realizations preserve the prescribed
motion, while UA reaches approximately $0.27\,\mathrm{m}$ tracking
error following the changes in required-force direction. The attitude
panel attributes this error to the finite UA reorientation time:
FA-opt generates the corrective force immediately and subsequently
approaches the minimum-allocation attitude. In agreement with
\cref{eq:zero-transverse-force}, the UA transverse body-force
coordinate is zero to numerical precision. FA-opt uses transverse force
only during alignment transients, whereas FA-const uses it persistently.
The generated-force cost must be read together with the tracking error,
since a lower cost can result from failure to realize the commanded
force.
Using the same required-force trajectory for all three realizations, rather than their actual generated forces, gives cumulative allocation costs of $3787$ for UA, $4355$ for FA-opt, and $8159$ for FA-const.
Thus, the simulation illustrates the complementary conclusions of the theory:
minimum-direction alignment reduces allocation cost, while full
actuation permits immediate departure from that family when preserving
tracking requires transverse force.

\begin{figure}[t]
 \centering
 \includegraphics[width=\columnwidth]{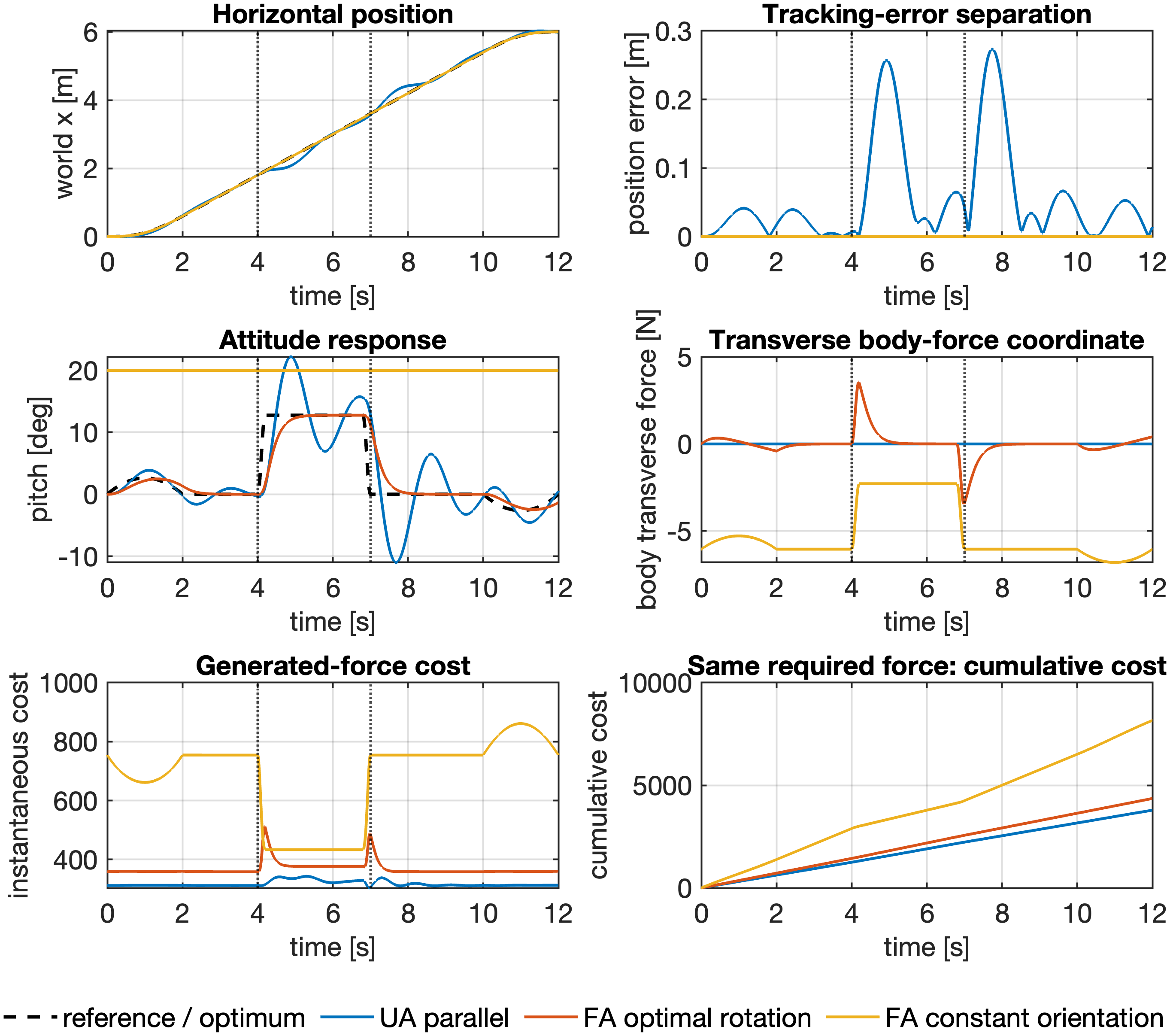}
 \caption{Planar comparison of the three actuation realizations. The
 dotted vertical lines delimit the external-force interval. The
 generated-force cost reflects the closed-loop force actually applied,
 whereas the cumulative cost evaluates the common force required by
 the prescribed motion and therefore isolates the allocation effect of
 attitude selection.}
 \label{fig:numerical-comparison}
\end{figure}

\section{Discussion and Conclusions}
\label{sec:discussion-conclusion}

This work distinguishes complete wrench capability from its
minimum-cost use. For a fixed actuation geometry, zero-moment
constraint, and actuator cost, the least costly body-force directions
are determined jointly. Positive homogeneity makes these directions
independent of force magnitude. Hence, for any prescribed position
trajectory, aligning a fixed minimum-allocation direction with the
required world-frame force confines the body force to a one-dimensional
subspace and eliminates its transverse component. Underactuated-like
force generation thus emerges from optimality rather than from limited
force authority.

The exact optimality-gap decomposition extends this result to full-pose
trajectories. Its minimum class contains precisely the trajectories
whose body-frame required-force direction remains in
$\minimumDirectionSet$ almost everywhere, while the accumulated gap
quantifies the additional allocation cost of every other attitude
trajectory.
For summed nonnegative actuator output, parallel thrust attains the
universal force-generation lower bound whenever alignment and zero
moment are feasible.

Full actuation retains a complementary role during changes in the
required-force direction. A fully actuated platform can immediately use
transverse force and subsequently recover minimum-cost alignment,
whereas a structurally underactuated platform must first reorient.
Nominal efficiency can therefore favor underactuated-like motion while
complete wrench capability supports tracking resilience.

The results concern instantaneous force allocation under a zero
body-moment constraint. Including the moment and actuator dynamics
required for attitude motion couples allocation cost to attitude
derivatives and requires trajectory-level optimization. Complete rigid-body dynamics, finite rotational
inertia, and nonhomogeneous physical-power models are left for future
work.

\bibliographystyle{IEEEtran}
\bibliography{refs}

\end{document}